\documentclass[twoside]{article}
\usepackage[colorlinks=true,linkcolor=magenta,citecolor=magenta]{hyperref}
\usepackage[letterpaper, margin=1in]{geometry}

\title{Regularized Training of Intermediate Layers for Generative Models for Inverse Problems}

\usepackage[utf8]{inputenc}
\usepackage[T1]{fontenc}
\usepackage{authblk}
\usepackage[numbers]{natbib}

\author[1]{Sean Gunn}
\author[2]{Jorio Cocola}
\author[1,2]{Paul Hand}
\affil[1]{Khoury College of Computer Sciences, Northeastern University}
\affil[2]{Department of Mathematics, Northeastern University}
\date{}

\usepackage{microtype}
\usepackage[pdftex]{graphicx}
\usepackage{subfigure}
\usepackage{booktabs} 
\usepackage{sectsty}
\usepackage{float}
\usepackage{hyperref}
\sectionfont{\fontsize{12}{15}\selectfont}

\usepackage{hyperref}

\usepackage{amsmath}
\usepackage{amssymb}
\usepackage{mathtools}
\usepackage{amsthm}
\usepackage{amsfonts,amsmath,bbm,amssymb, amsthm,bm}

\usepackage{xcolor}
\usepackage{soul}
\usepackage{import}
\usepackage{svg}
\usepackage{comment}
\usepackage{float}
\usepackage{algorithm}
\usepackage[noend]{algpseudocode}

\svgsetup{inkscapelatex=false}

\DeclareMathOperator*{\argmin}{arg\,min}
\newcommand{\R}{\mathbb{R}}
\newcommand{\I}{{I}}

\newcommand{\E}{\mathbb{E}}
\newcommand{\Gtil}{\widetilde{G}}
\newcommand{\G}{G}

\newcounter{alphasect}
\def\alphainsection{0}

\let\oldsection=\section
\def\section{%
  \ifnum\alphainsection=1%
    \addtocounter{alphasect}{1}
  \fi%
\oldsection}%

\renewcommand\thesection{%
  \ifnum\alphainsection=1%
    \Alph{alphasect}
  \else%
    \arabic{section}
  \fi%
}%

\newenvironment{alphasection}{%
  \ifnum\alphainsection=1%
    \errhelp={Let other blocks end at the beginning of the next block.}
    \errmessage{Nested Alpha section not allowed}
  \fi%
  \setcounter{alphasect}{0}
  \def\alphainsection{1}
}{%
  \setcounter{alphasect}{0}
  \def\alphainsection{0}
}%

\newcommand{\Wstar}{W^{\star}}

\newcommand{\ztil}{\widetilde{z}}
\newcommand{\zhat}{{z}}
\newcommand{\Wvan}{W^{\text{Van}}}
\newcommand{\Wrtil}{W^{\text{RTIL}}}
\newcommand{\Gvan}{G^\text{Van}}
\newcommand{\Gvantil}{\widetilde{G}^{\text{Van}}}
\newcommand{\Grtil}{\widetilde{G}^\text{RTIL}}
\newcommand{\zstar}{z^{\star}}
\newcommand{\zvan}{z^\text{van}}
\newcommand{\zrtil}{z^\text{rtil}}

\newcommand\blfootnote[1]{%
  \begingroup
  \renewcommand\thefootnote{}\footnote{#1}%
  \addtocounter{footnote}{-1}%
  \endgroup
}

\usepackage[capitalize,noabbrev]{cleveref}

\theoremstyle{plain}
\newtheorem{theorem}{Theorem}[section]

\newtheorem{lemma}[theorem]{Lemma}

\theoremstyle{definition}

\theoremstyle{remark}

\usepackage[textsize=tiny]{todonotes}
\begin{document}
\twocolumn
\maketitle
\begin{abstract}
Generative Adversarial Networks (GANs) have been shown to be powerful and flexible priors when solving inverse problems. One challenge of using them is overcoming representation error, the fundamental limitation of the network in representing any particular signal. Recently, multiple proposed inversion algorithms reduce representation error by optimizing over intermediate layer representations. These methods are typically applied to generative models that were trained agnostic of the downstream inversion algorithm. 
In our work, we introduce a principle that if a generative model is intended for inversion using an algorithm based on optimization of intermediate layers, it should be trained in a way that regularizes those intermediate layers. We instantiate this principle for two notable recent inversion algorithms: Intermediate Layer Optimization and the Multi-Code GAN prior. For both of these inversion algorithms, we introduce a new regularized GAN training algorithm and demonstrate that the learned generative model results in lower reconstruction errors across a wide range of under sampling ratios when solving compressed sensing, inpainting, and super-resolution problems. 
\end{abstract}
\blfootnote{Sean Gunn <gunn.s@northeastern.edu>}
\blfootnote{Jorio Cocola <cocola.j@northeastern.edu>}
\blfootnote{Paul Hand <p.hand@northeastern.edu>}
\blfootnote{Github Link : \href{https://github.com/g33sean/RTIL}{https://github.com/g33sean/RTIL}    }
\vspace{-.5cm}
\section{Introduction}
\vspace{-.2cm}
The task in an inverse problem is to estimate an unknown signal given a (possibly noisy) set of measurements of that signal. 
In practice, inverse problems are often ill-posed, which often require the incorporation of prior information about the target signal to recover a reasonable estimate of it. 
 
Deep generative models have demonstrated remarkable performance when used as priors for solving inverse problems \citep{bora2017compressed, athar2018latent, NEURIPS2018_1bc2029a, menon2020pulse,mardani2018deep, mosser2020stochastic,pan2021do}. 
Typically, a generative modeling-based approach for inverse problems has two phases: a training phase and an inversion/deployment phase. In the first phase, a generative network is trained on a dataset of images different from the data in the testing phase. After training, the network parameters are typically fixed and an optimization problem involving the known forward model is solved to estimate an unknown
signal of interest. An advantage of this class of methods is that the prior can be trained entirely in an unsupervised manner and without previous knowledge of the specific inverse problem that needs to be solved downstream. This allow the trained network to be used for a variety of inverse problems. 
 
Some generative networks explicitly map a low-dimensional input to a high-dimensional signal space. The outputs of the generative network therefore have a low intrinsic dimensionality. While this is useful for  regularizing the inverse problem, it can limit the expressivity of the network and lead to \textit{representation error}, the error between the target signal and the closest signal in the range of the generative network.
 
Several recent methods have attempted to reduce representation error by introducing optimization algorithms that enlarge the search space of the inversion algorithm. These include methods that optimize both over the input of the network as well as over the activations of one or more intermediate layers. Examples of this class of inversion methods include  Intermediate Layer Optimization (ILO) \citep{ILO2021}, GAN Surgery \citep{smedemark2021generator}, and the Multi-code GAN Prior (mGANprior) \citep{Gu_2020_CVPR}. 
These algorithms were applied to state-of-the-art off-the-shelf GANs, such as PGGAN and StyleGAN \citep{karras2018progressive,karras2020analyzing}, which were trained agnostic to the specific optimization algorithm that would be used for inversion.  Consequently, the inversion algorithms optimize over a regions of the space of intermediate presentations that were not regularized during training. This suggests that it may be possible to improve the reconstruction performance of these algorithms by explicitly regularizing the intermediate layers during the training phase. In this paper we confirm this hypothesis.   

 
 We consider the problem of training a GAN that will be used for solving inverse problems via an inversion algorithm that optimizes over intermediate layers. We put forward the following principle termed \textit{Regularized Training of Intermediate Layers (RTIL)}:
\begin{center}
    \textit{If a GAN's intermediate layers are optimized during inversion, they should be regularized during training.}
\end{center}

We use this principle to derive new training algorithms for GANs that are intended for solving inverse problems. For both ILO and the mGANprior, we introduce a new training algorithm for StyleGAN and PGGAN, respectively. For these new trained GANs, we achieve improved reconstruction quality relative to the corresponding GANs trained without the principle.

This principle may lead to the following workflow for solving inverse problems with GANs. First, train a generative network with latent variable $z_0$, sampled from a latent distribution $p_{z_0}$, and outputting signals $x$ from the target distribution (Figure \ref{fig:rtil_princ} Step 1). Second, explore various inversion algorithms, including some that optimize over intermediate layers by introducing an
additional optimization variable $z_1$ (Figure 1 Step 2).  Algorithms of this type include ILO and mGANprior. Third, if such an algorithm provides competitive performance for inversion, then use RTIL to devise a new GAN training algorithm.  This can be achieved by introducing a new latent variable $z_1 \sim p_{z_{1}}$ where  $p_{z_{1}}$ is an appropriate probability distribution  (Figure \ref{fig:rtil_princ} Step 3). 
Finally, during deployment, use this trained generative network for inverse problems via the selected inversion algorithm (Figure \ref{fig:rtil_princ} Step 4).

This workflow demonstrates a way to use recently introduced inversion algorithms in order to inspire new training algorithms for GANs. It provides additional ways of training GANs knowing that they will be used for inversion, and it provides a way to ensure that some empirically successful inversion algorithms are operating in a more principled manner by ensuring they are searching over a space of parameters that has been suitably regularized.

\begin{figure}[H]
\begin{center}
\includegraphics[scale=0.22]{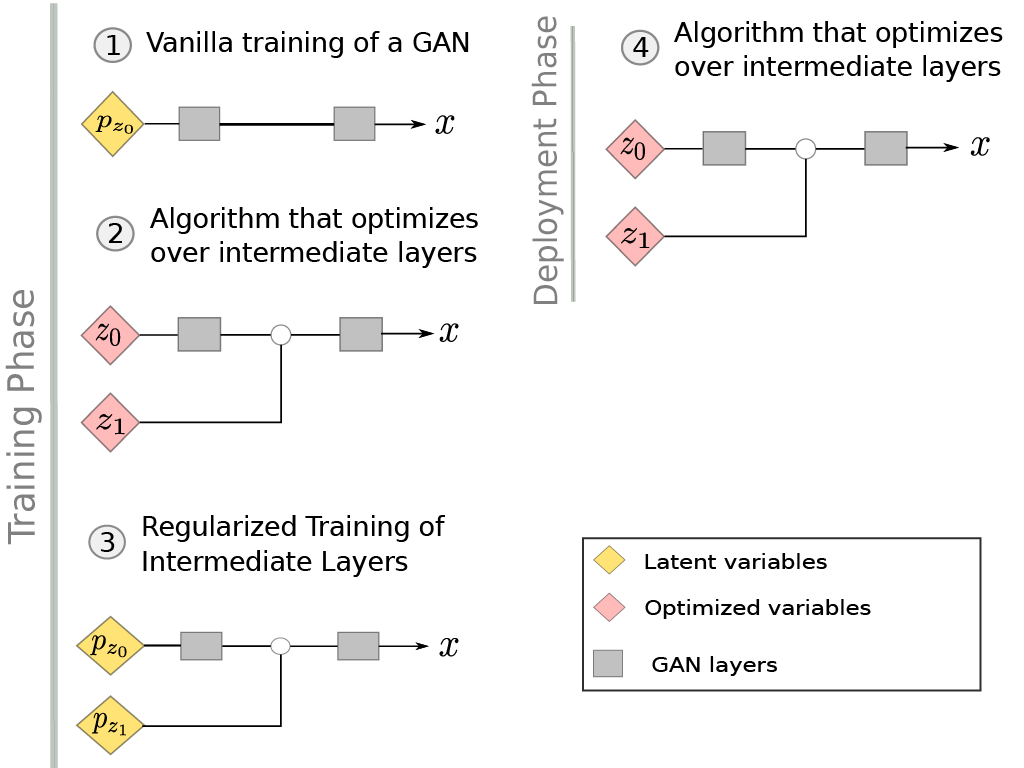}
\end{center}
\caption{The workflow for using Regularized Training of Intermediate Layers (RTIL) to develop new training algorithms of GANs that will be used for inversion.  See the text for details.}
\label{fig:rtil_princ}
\end{figure}


The contributions of this paper are the following. 
\begin{itemize}

\item We introduce Regularized Training of Intermediate Layers, a principle for training deep generative networks that are intended to be used for inverse problems with an optimization algorithm that optimizes over intermediate layer representations.

\item In the case of Intermediate Layer Optimization, we use our principle to devise a novel GAN training algorithm.  With the resulting trained GAN, we demonstrate lower reconstruction errors (compared to GAN training without the principle) for compressed sensing, inpainting, and super resolution over a wide range of under sampling ratio.

\item  We show the versatility of the method by repeating the same contribution in the case of the Multi-Code GAN Prior. 

\item We illustrate the benefits of compressed sensing with RTIL by theoretically showing that a model trained without regularizing intermediate layers has a strictly larger reconstruction error than a model where the intermediate layers where regularized during training.  For simplicity, this result is established in the case of a two layer linear neural network trained in a supervised setting.


\end{itemize}

\subsection{Related Works}



In the recent years, various approaches have been proposed for dealing with the representation error when solving inverse problems with a learned generative prior. As described above, there is one class of approaches that focus on enlarging the latent space dimension either at training or at inversion time \citep{athar2018latent,dhar2018modeling,hussein2020image}. On the other hand, there is another notable class of approaches recently put forward is based on flow-based invertible neural networks. These generative networks have invertible architectures and a latent space with the same size of the image space, and thus have zero representation error \citep{ardizzone2018analyzing,asim2020invertible,ma2009deblurring, kelkar2021compressible,shamshad2019subsampled,whang2021solving,whang2021composing,helminger2021generic}. Both these type of approaches attempt to limit the representation error during training or inversion. 

Similarly to flow-based models, score networks and subsequent variants \citep{song2019generative,song2020improved,song2020score}  have support on the entire signal space, allow for conditional sampling and likelihood estimation, and have been object of recent interests for their use in inverse problems \citep{ramzi2020denoising,jalal2021robust,jalal2021instance}.

The increased representation power of the above mentioned generative networks, comes usually at a price of increased computational cost both during the training and the inversion phase. In contrast, our proposed principle leads to training algorithms that have essentially the same computational cost as the standard ones, also leaving the cost of the inversion algorithms the same.


\section{Generative Models and Optimization Algorithms for Inverse Problems }
\subsection{Background}
\label{sec:background}
We consider the problem of training a GAN for the use of a prior for solving inverse problems. We consider a GAN $G$ which maps a latent space $\mathbb{R}^{n_{0}}$ to an image space $\mathbb{R}^{n_{d}}$, where $n_{0} \ll n_{d}$. In this paper, we focus on the linear imaging inverse problems of compressed sensing, inpainting, and superresolution, though our proposed method also applies to nonlinear inverse problems such as phase retrieval, and inversion problems about non-image signals. We consider the general linear inverse problem of recovering an image $x \in \R^{n_{d}}$ from a set of  linear measurements $y \in \R^{m}$, given by $y = \mathcal{A}(x)$, where 
$\mathcal{A}: \R^{n_{d}} \to \R^{m}$ is a forward linear measurement operator. We study only the noiseless case, but our method easily extends to case with measurement noise. In this paper we study measurement operators $\mathcal{A}$ of the following form: 

\begin{itemize}
\item  Compressed Sensing: $\mathcal{A}= A \in \R^{m \times n_{d}}$, $A$ is a random matrix that samples from a known distribution with $m < n_{d}$. 
\item Inpainting: $\mathcal{A} = M \in \R^{m \times n_{d}}$,  $M$ is a masking matrix with binary entries.
\item Super-Resolution: $\mathcal{A} =S_{m_{\downarrow}} \in \R^{m \times n_{d}}$ where $S_{m_{\downarrow}}$ is the downsampling operator with downsampling factor $m_{\downarrow}$. 

\end{itemize}
An estimate of $x$ can be recovered by finding the image in the range of $G$ that is most consistent with the measurements $y$ in the following sense, as introduced in \citep{bora2017compressed}.  First, solve
\begin{align}
\hat{{z}}_{0} = \underset{z_{0}}{\argmin} \hspace{2mm} \|y- \mathcal{A}(G(z_{0})) \|,
\end{align}
then, the estimate of $x_0$ is given by $G(z_{0})$. 

As mentioned in the introduction, a difficulty of this optimization approach is that the estimated images are constrained to live within the range of $G$, which is a $n_{0}$-dimensional manifold in $\R^{n_{d}}$.  Most images $x_0$ will not live exactly in this range, and thus the method is limited by the representation error $\min_{z_{0}} \| x_0 - G(z_{0}) \|$.  

In the next sections we review two recent algorithms for mitigating representation error during inversion. For ease of exposition, we will discuss the case where only one intermediate layer representation is optimized.  We write $G = g_1 \circ g_0$  where $g_0: \R^{n_0} \to \R^{n_1}$ and $g_1: \R^{n_1} \to \R^{n_d}$.


\subsection{Intermediate Layer Optimization (ILO)~\cite{ILO2021}}
\label{sec:ILO}
Given a  trained generative model $G = g_1 \circ g_0$, Intermediate Layer Optimization (ILO) extends the range of the generative model by sequentially optimizing over each layer of the network, as demonstrated in Algorithm \ref{alg:ILO}.The initial step begins exactly as in \citep{bora2017compressed} by optimizing over the input vector $ z_{0} \in \R^{n_{0}}$. The solution is obtained in line 3, by initializing a $z_{0} \sim \mathcal{N}(0,\I_{n_{0}})$, then optimizing the loss with gradient descent. After the solution $\hat{z}_{0}$ is obtained, 
the algorithm searches for a perturbation $\hat{z_1}$ of $g_{0}(\hat{z}_{0})$ that further minimizes the reconstruction error (line 4).  The final approximation of the target image $x$ is given by $g_{1}(\hat{z}_{1} + g_{0} (\hat{z}_{0}))$.  As the authors point out, there are multiple ways where ILO can be regularized, including by an L1 penalty in the intermediate representation, or via early stopping. For the present paper, we present the method with early stopping as this was the method used by the publicly available code from the authors.  Throughout this paper, we use the code provided by the authors when solving ILO. 

\begin{algorithm}
\caption{Intermediate Layer Optimization (ILO) for Compressed Sensing \citep{ILO2021}.}
\begin{algorithmic}[1]
\State {\bfseries Input:} $G= g_{1} \circ g_{0}$, measurement matrix $\displaystyle A \in \R^{m \times n_{d}}$, compressed measurements $\displaystyle y$.
\State {\bfseries Output: } estimated image $\hat{x} $
\State $ \displaystyle \hat{z}_{0} = \underset{z_{0}}{\argmin} \hspace{2mm} \|y-A g_{1} \big( g_{0} (z_{0}) \big)\| $ \textit{Initialize at} $z_{0} \sim \mathcal{N}(0,\I_{n_{0}})$
\State $\hat{z}_{1} = \underset{z_{1}}{\argmin} \hspace{2mm} \|y-A g_{1}(z_{1} + g_{0} (\hat{z}_{0}) )\| $ 
\State {Return:} {$ \hat{x} = g_{1}(\hat{z}_{1} + g_{0} (\hat{z}_{0})  )$}
\end{algorithmic}
\label{alg:ILO}
\end{algorithm}

\begin{figure}[h!]
	\begin{center}
	\includegraphics[width=1\linewidth]{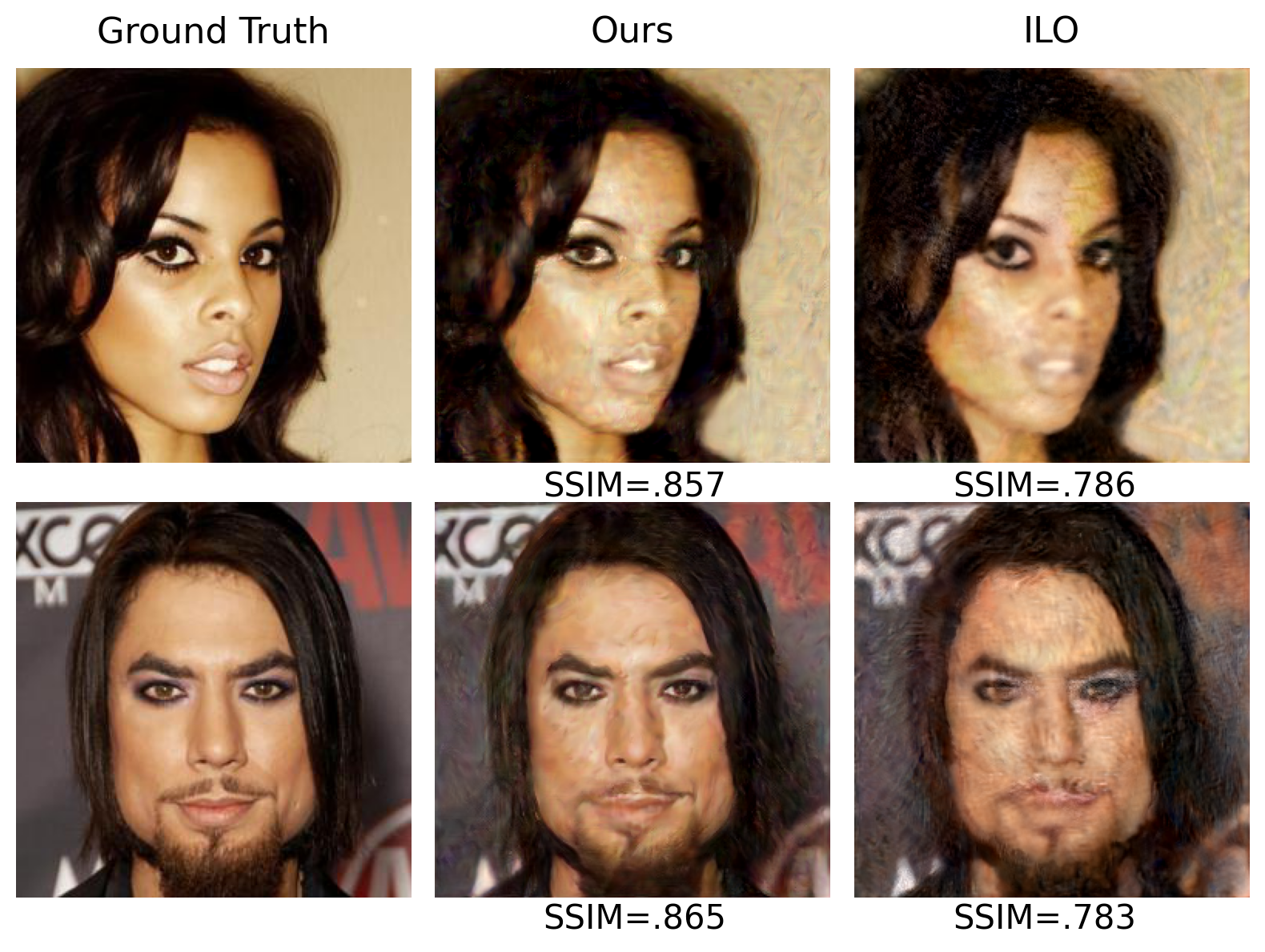}
	\end{center}
	\vspace{-.5cm}
	\caption{Comparison between ILO-RTIL (ours) and ILO for compressed sensing for 3\% of measurements. }
	\label{ILO_CS}
\end{figure}

\subsection{Multi-Code GAN (mGanPrior)}
\label{sec:mGANpiror}

Multi-Code GAN Prior is an inversion method that simultaneously optimizes over multiple latent codes and composes their corresponding intermediate feature maps with adaptive channel importance. This effectively extends the expressivity of the network by giving it a higher dimensional input space and additional parameters to control the importance of each channels in the intermediate layer representation. 

Assume we are given a pre-trained generative model $G = g_1 \circ g_0$, where the output of the first layer $g_0$ has dimension $H_1 \times W_1 \times C_1$ with $C_1$ being the number of channels.  Furthermore, chose $N$ latent codes $\{z_0^k\}_{i=1}^N \in \R^{n_0}$ and channel importance $\{\alpha^k\}_{i=1}^N \in \R^{C_1}$. Then the Multi-Code GAN Prior extended architecture computes $g_{1}(\sum_{k=1}^{N}g_{0}({z}_{0}^{k}) \odot {\alpha}^{k})$ where $\{g_{0}({z}_{0}^{k}) \odot {\alpha}^{k}\}_{ijc} = \{g_{0}({z}_{0}^{k})\}_{ijc} \cdot  \{\alpha^{k}\}_c$ is channel-wise multiplication, $i, j$ are spatial location, and $c$ is the channel index. 

During the inversion phase the Multi-Code GAN Prior method optimizes over both the latent vectors $\{z_0^k\}_{i=1}^N$ and the channel importance $\{ \alpha^k \}_{i=1}^N$ (Algorithm \ref{alg:mGAN}).




\begin{algorithm}
\caption{Multi-Code GAN (mGANprior) \cite{Gu_2020_CVPR}.}
\label{alg:mGAN}
\begin{algorithmic}[1]
\label{ilo_alg}
\State {\bfseries Input:} Trained network $G= g_{1} \circ g_{0}$, latent codes $\{z_{0}^{k}\}_{k=1}^{N} \in \R^{n_0} $ , $\{\alpha^{k}\}_{k=1}^{N} \in \R^{n_1} $, measurement matrix $A \in \R^{m \times n_{d}}$, compressed measurements $y$.
\State {\bfseries Output: } estimated image $\hat{x} $
\State \textit{Initialize } $ \{z_{0}^{k}\}_{k=1}^{N} \sim p_{z}, \{\alpha^{k}\}_{k=1}^{N}  \sim p_{\alpha}$ 
\State $ \hat{z},\hat{\alpha} = \hspace{-3mm} \underset{\{z_{0}^{k}\}_{k=1}^{N} , \{\alpha^{k}\}_{k=1}^{n} } {\argmin} \hspace{-3mm}\| y - Ag_{1}(\sum_{k=1}^{n}g_{0}(z_{0}^{k}) \odot \alpha^{k}) \| $
\State {Return:} {$ \hat{x} = g_{1}(\sum_{k=1}^{N}g_{0}(\hat{z}_{0}^{k}) \odot \hat{\alpha}^{k})$}
\end{algorithmic}
\end{algorithm}

\vspace{-.75cm}
\section{RTIL}
\vspace{-.5cm}

\begin{figure}[h!]
	\begin{center}
	\includegraphics[width=1.0\linewidth]{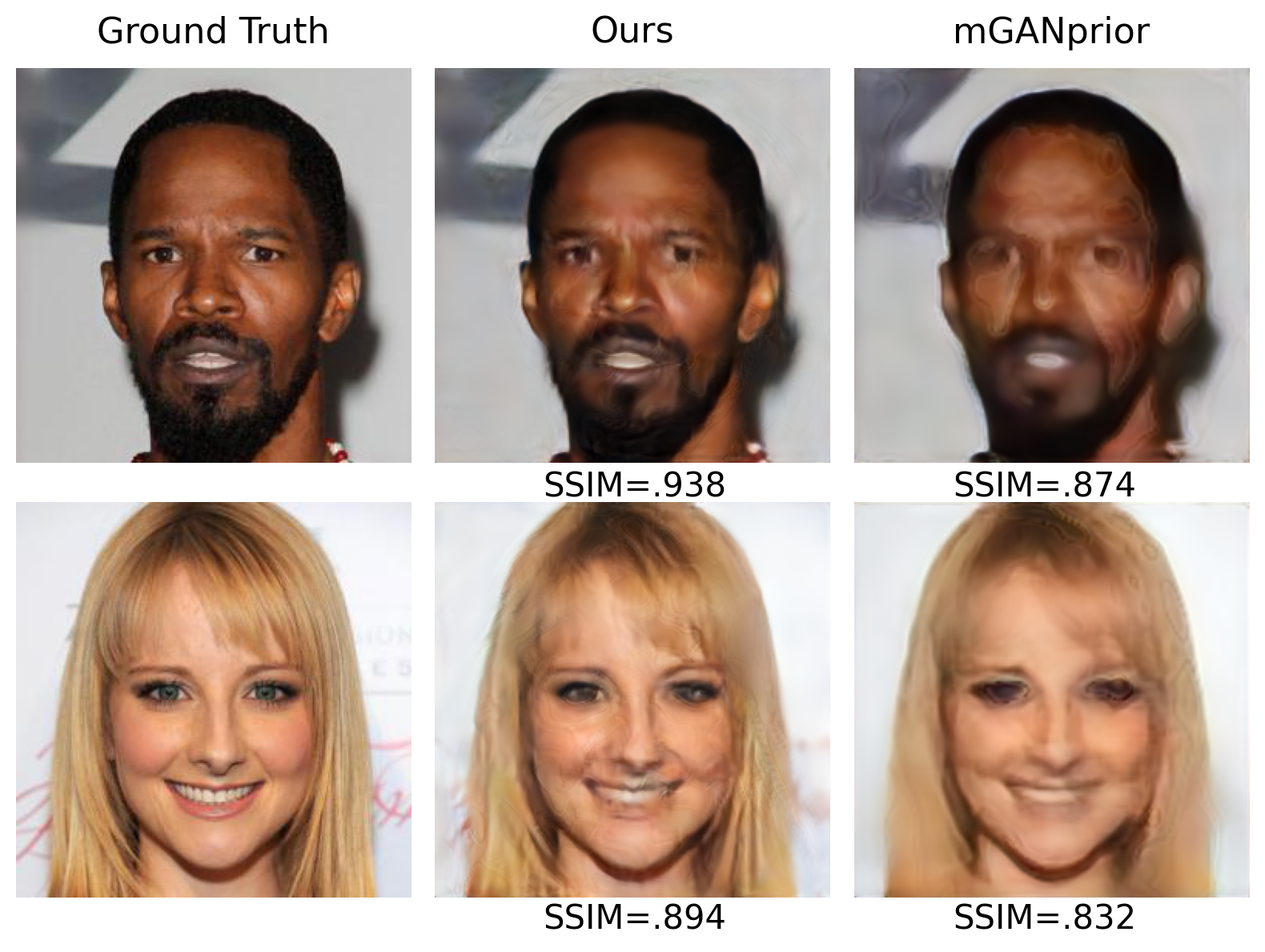}
	\end{center}
	\vspace{-.5cm}
	\caption{Comparison between mGANprior-RTIL (ours) and mGAN for compressed sensing  for 5\% of measurements. }
	\label{mGANprior_cs}
\end{figure}
\vspace{.25cm}
In this section, we present how the principle of   Regularized Training of Intermediate Layers (RTIL) inspires training algorithms for GANs that are intended for inversion by Intermediate Layer Optimization (Section \ref{subsec:RTIL-ILO}) and  mGANprior (Section \ref{subsec:RTIL-mGANprior}).
  

We consider the case of a practitioner having chosen, after some initial exploration, a base generative network $\G$ for use as prior in solving inverse problems. For simplicity in this section we will consider $\G = g_1 \circ g_0$ where $g_0 : \R^{n_0} \to \R^{n_1}$ and $g_1: \R^{n_1} \to \R^{n_d}$. 
The input latent vectors $z_0$ of the network $G_{\theta}^{0}$ are sampled from $p_{z_{0}}$ (e.g. $p_{z_0} = \mathcal{N}(0,\I_{n_{0}})$) and $\theta$ is the set of trained parameters. 

We assume, furthermore, that the practitioner has selected an inversion algorithm that optimizes over the latent variable $z_0$ and the intermediate layer between $g_0$ and $g_1$ of $\G$. Optimizing over the intermediate layer corresponds to introducing a free variable $z_1$ between $g_0$ and $g_1$  (Figure \ref{fig:rtil_general} left).

The RTIL principle states that if one intends to solve inverse problems by means of intermediate layer optimization algorithms, then 
\textit{intermediate layers optimized over during inversion should be regularized during training}.

This principle can be used to design a new training algorithm in the following manner. We identify the additional free variable, $z_1$, used for optimization over the intermediate layer, consider it as a latent variable of the generative model, and provide it with a simple distribution $p_{z_{1}}$. For example, this could result in a generative model $\Gtil : \R^{n_0 \times n_1} \to \R^{n_d}$, such that $\Gtil (z_0 , z_1) = g_1 (z_1 + g_0(z_0))$, but could have alternative functional forms. This generative model reduces to the base generative model $\G: \R^{n_0} \mapsto \R^{n_d} $ if $z_1$ is suitably chosen, for example if $z_1=0.$ 
. 

The introduction of the latent variable $z_1$ explicitly increases the dimensionality of the latent space. In practice, training latent variable models with high-dimensional latent spaces can be challenging and require careful regularization ~\citep{athar2018latent}. We address this difficulty by concurrently training the lower and higher dimensional models $\G_{\theta}$ and $\Gtil_{\theta}$, which share trainable weights $\theta$ (Figure \ref{fig:rtil_general} right).

We train $\G_{\theta}$ and $\Gtil_{\theta}$ via the following minimax formulation~\citep{NIPS2014_5ca3e9b1}
\hspace{-2cm}
\begin{align*}
\min_{\theta} &\max_{\Theta} \,\E_{ x \sim p_{x}} [ \log D_\Theta(x) ] \\ 
&+ \frac{1}{2} \E_{\substack{ z_{0} \sim p_{z_{0}} \\  z_{1} \sim p_{z_{1}} }}   
\big[  \log(1-D_\Theta(\G_{\theta}({z_{0}})) \\ 
&\qquad+ \log(1-D_\Theta(\Gtil_{\theta}({z_{0}},{z_{1}})) \big],
\end{align*}
where $D_\Theta: \R^{n_d} \to \R$ is the discriminative network.  This method could be extended to alternative GAN and non-GAN formulations ~\citep{arjovsky2017wasserstein,gulrajani2017improved,bojanowski2018optimizing}.


Once $\Gtil$ is trained, a practitioner  could solve an inverse problem with forward operator $\mathcal{A}$ by solving
\begin{equation*}
\hat{{z}}_{0}, \hat{{z}}_{1} = \underset{z_{0}, z_1}{\argmin} \hspace{2mm} \|y_{0}- \mathcal{A}(\Gtil (z_{0}, z_{1})) \|,
\end{equation*}
using the selected optimization algorithm, resulting in $\Gtil (\hat{z}_{0}, \hat{z}_{1})$ as the estimate of the signal.

In this section, we considered the case where only one intermediate layer was optimized.  The proposed method can directly extend to the case where multiple layers are optimized.


\begin{figure}
\includegraphics[width=\linewidth,height=4cm]{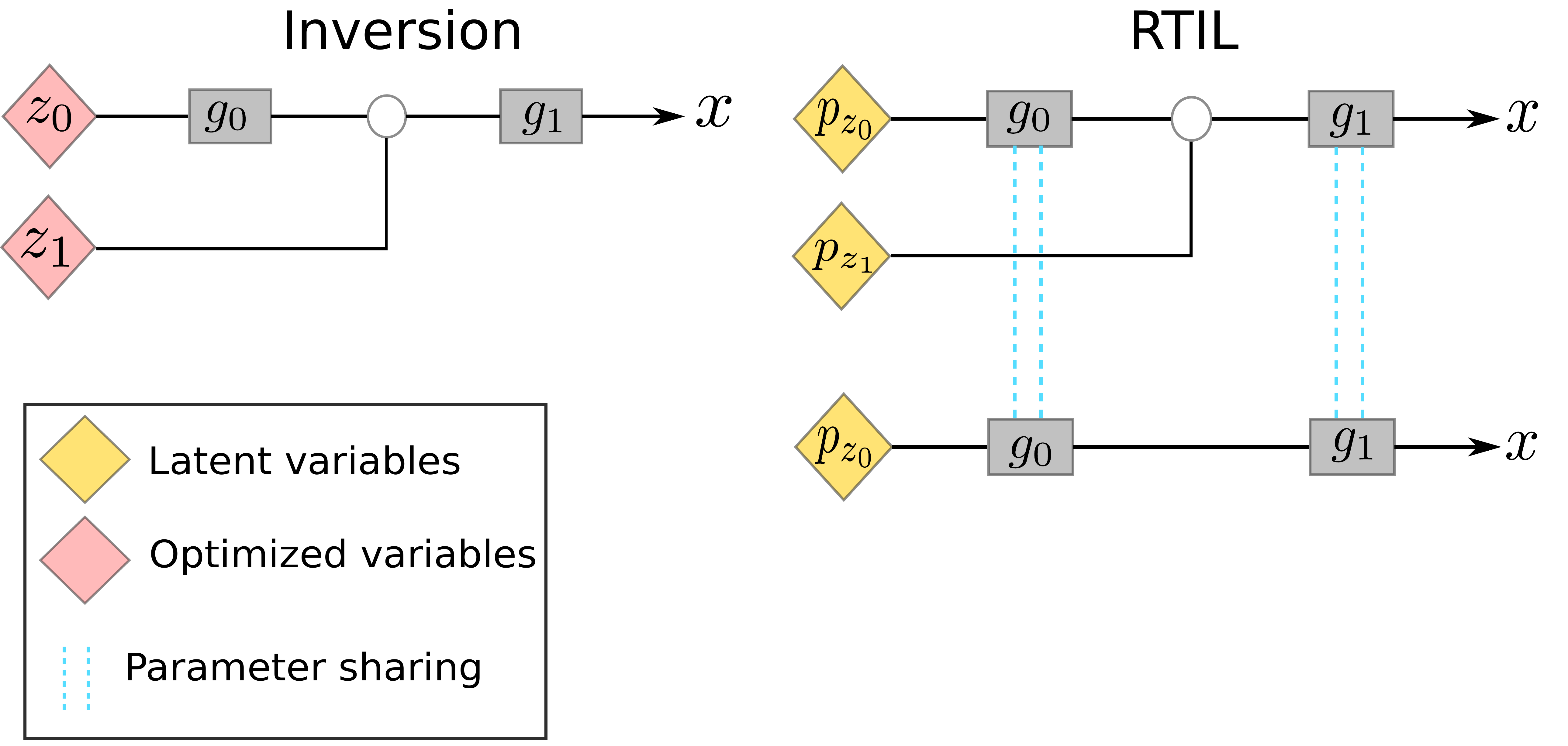}
\vspace{-.5cm}
\caption{Visual representation of the RTIL principle.  An  inversion method that optimizes over intermediate layers informs the training of a family of generative networks adapted to that inversion method.}
\label{fig:rtil_general}
\end{figure}

We will next describing the details of the application of RTIL in the case of Intermediate Layer Optimization and the Multi-Code GAN Prior.

\subsection{RTIL for ILO}\label{subsec:RTIL-ILO}

\begin{figure*}[h!]
\includegraphics[width=\linewidth]{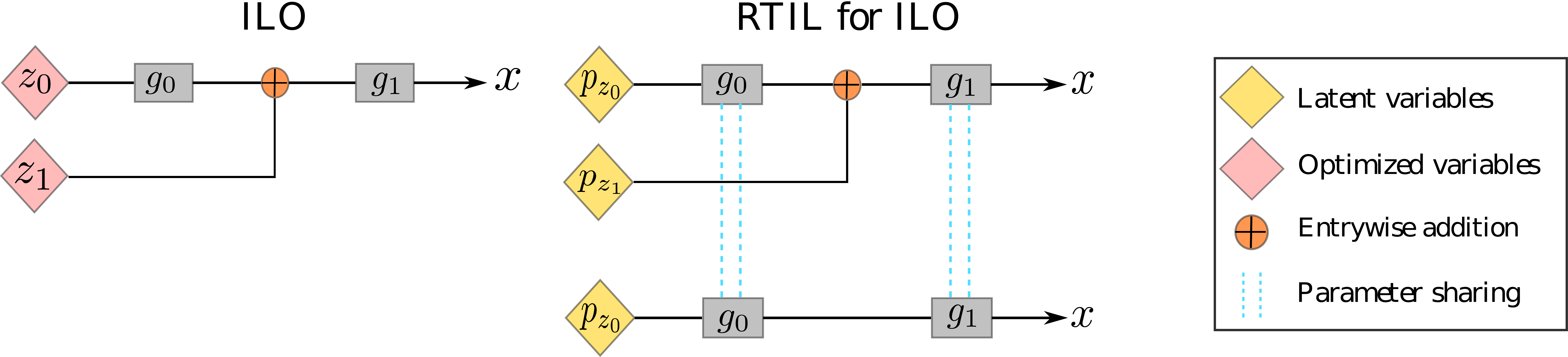}
\vspace{-.5cm}
\caption{ The left side portrays ILO, and the right side demonstrates how this inspire the training of a family of generative models with ILO-RTIL with parameters sharing.}
\label{fig:rtil_ILO}
\end{figure*}

In the case of Intermediate Layer Optimization (ILO), we now present how to use RTIL to design a GAN training algorithm.  Consider the base generative network $\G (z_0) =g_{1} ( g_{0} (z_0))$, to be used for inversion with Intermediate Layer Optimization. ILO (Algorithm \ref{alg:ILO}) extends the range of the network by optimizing  $g_{1}(z_{1} + g_{0} ({z}_{0}) )$ over latent variables $z_0$ and $z_1$. Consequently, we consider the higher dimensional model generative model
\begin{align*}
    \Gtil(z_0, z_1) &= g_{1}(z_{1} + g_{0} ({z}_{0}) )\\
    z_0 &\sim \mathcal{N}(0, I_{n_0})\\
    z_1 &\sim \mathcal{N}(0, \sigma^{2} I_{n_1}),
\end{align*}
where $\sigma^2$ is a hyperparameter. Then simultaneously the lower dimensional model $\G(z_{0})$ and higher dimensional model $\Gtil(z_{0},z_{1})$ are trained with the min max formulation above. Note $z_0 \sim p_{z_0}$ and  $z_1 \sim p_{z_1}$ where Figure \ref{fig:rtil_ILO} depicts this process.

\subsection{RTIL for mGANprior}\label{subsec:RTIL-mGANprior}

In the case of Multi-Code GAN Prior method, we now present how to use RTIL to design a GAN training algorithm.  Consider the base generative network $G^{0}_{\theta}$
to be used for inversion with the Multi-Code GAN Prior method. The mGANprior Algorithm (Algorithm \ref{alg:mGAN}) extends the range of the network by optimizing  $g_{1}(\sum_{k=1}^{N}g_{0}({z}_{0}^{k}) \odot {\alpha}^{k})$ over latent variables $z_0$ and $z_1$. Consequently, training the higher dimensional model yields
\begin{align*}
\Gtil ({z_0^1}, \dots, {z_0^N}, {\alpha^1}, \dots, {\alpha^N }) & =  g_{1}(\sum_{k=1}^{N}g_{0}({z}_{0}^{k}) \odot {\alpha}^{k}) \\
z^{k}_0 &\sim \mathcal{N}(0, I_{n_0})\\
p_{\alpha'} &\sim \text{Dir}_N(1)
\end{align*}

 Specifically,  a vector $\alpha' \in \R^{N}$ is sampled from $p_{\alpha'} \sim \text{Dir}_N(1)$, where $\text{Dir}_N(1)$ is the flat Dirichlet distribution, i.e.\ the uniform distribution over the $(N-1)$-dimensional simplex. Then each vector $\alpha^k$ is taken to be $\alpha^k = \{\alpha'\}_k \cdot \bm{1}$ where $ \bm{1}$ is the vector of all ones and $ \{\alpha'\}_k$ is the $k$-th entry of the vector $\alpha'$. Note this leads to each channel being weighted equally during training. This process can be depicted in the appendix Figure \ref{fig:rtil_mgan}.

\vspace{-.3cm}
\section{Experiments}
\vspace{-.15cm}
\begin{figure*}[h!]
	\begin{center}
	\includegraphics[width=1.0\linewidth,height=6cm]{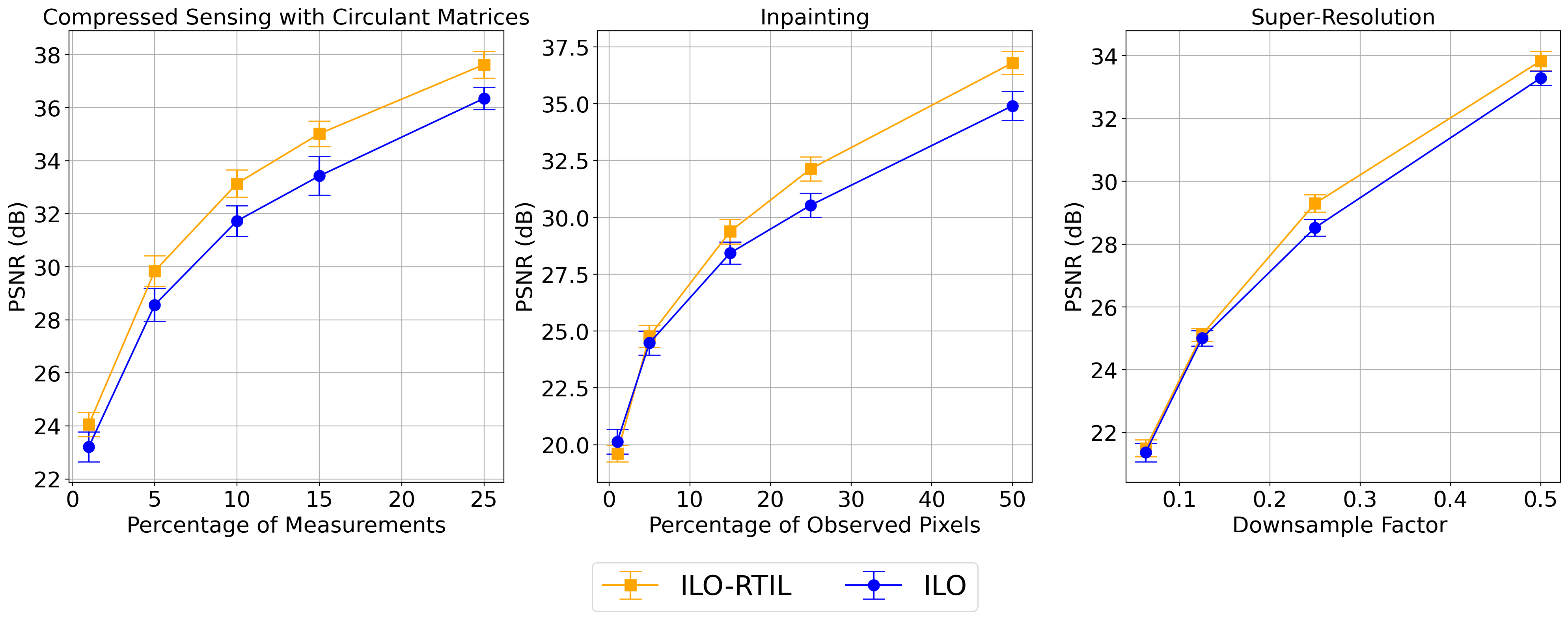}
	\end{center}
	\vspace{-.5cm}
	\caption{Performance of ILO-RTIL and vanilla trained ILO for Compressed sensing, inpainting, and super resolution for various under-sampling ratios. ILO-RTIL increases performances or ties for each under-sampling ratio with respect to PSNR across each inverse problems compared to ILO. The vertical bars indicate 95\% confidence intervals.  }
	\label{fig:RTIL-vs-ILO}
\end{figure*}

We observe that Regularized Training of Intermediate Layers is not tied to any specific architecture or training procedure, and while, until this point, we have only presented it on two-layer networks for easiness of exposition, in this section we demonstrate its successful application on two state of the art architectures on different imaging recovery problems. Specifically, we conduct extensive experiments comparing RTIL versus vanilla training for compressed sensing, inpainting, and super-resolution, and for two different inverse methods Multi-Code GAN Prior (Section \ref{sec:mGANpiror}) and Intermediate Layer Optimization (ILO) (Section \ref{sec:ILO}) to demonstrate the effectiveness of our method. The generative model architecture used for the mGANprior  is PGGAN \citep{karras2018progressive} and for ILO is styleGAN-2~\citep{karras2020analyzing}. All models were trained on FFHQ data set~\citep{karras2019style} and tested on CelebA-HQ dataset~\citep{karras2018progressive}, at an image size of 256x256x3. The choice of architectures were based on the experimental section of the original ILO \citep{ILO2021} and mGANprior \citep{Gu_2020_CVPR} papers. Refer to the appendix for the architecture details for training networks using RTIL, as well as hyperparameters chosen for inversion methods.

For all experiments, compressed sensing use partial circulant measurement matrices with random signs~\cite{ILO2021}, inpainting the pixels are missing at random, and downsample factor corresponds to how much the height and width of the original image was reduced to. 

\subsection{Results ILO-RTIL}
\label{sec:exp-ILO-RTIL}
Results in this section correspond to Figure \ref{fig:RTIL-vs-ILO}, where the results are average over three test sets of 12 images randomly sampled from CelebA-HQ. 

\underline{Compressed Sensing } - ILO-RTIL demonstrates an increase in reconstruction performance across each under-sampling regime. The largest increase performance with respect to PSNR occur at 25\% (1.28 dB), 15\% (1.59 dB), and  10\% (1.3 dB's) measurements. For qualitative results please refer to Figure \ref{ILO_CS}.
\underline{Inpainting}- ILO-RTIL demonstrates in increase in reconstruction performance with 4 out of the 5 sampling regimes. There is increase in PSNR at 50\% (1.88 dB), 25\% (1.59 dB's), 15\% (.94 db's), and 5\% (.29 db) of observed pixels. At 1\% of observed pixels there is a decrease in performance in reconstruction  of .5 dB's, but the error bars overlap each other indicating there is no significant improvement. For qualitative results please refer to Figure 7. \underline{Super-Resolution} -  ILO-RTIL shows a slight increase in performance over the entire measurement regime compared to ILO, the most significant occurs at $\frac{1}{4}$ downsampling factor where on average the increase in reconstruction is .77 dB's and the 95\% error bars are separated. All other sampling regime achieve comparable performance between ILO-RTIL and ILO. For qualitative results please refer to the appendix Figure~\ref{fig:ilo_sr}. 
\vspace{-.3cm}
\subsection{Results RTIL-Multi-Code}
\label{sec:RTIL-multi}
Results in this section correspond to Figure \ref{fig:RTIL-vs-mGAN}, where the results are average over five test sets of 12 images randomly sampled from CelebA-HQ. All experiments use $N=20$ latent codes
\underline{Compressed Sensing }- mGANprior-RTIL has a significant improvement in reconstruction performance over mGANprior. For each under-sampling ratio our method increases performance by at least 1.6 dB's, which clearly separates the error bars. On the other hand,  at  $1\%$ of measurements where both methods achieve the similar performance, which can seen by overlapping error bars. For qualitative results refer to Figure \ref{mGANprior_cs}.\underline{Inpainting}- mGANprior-RTIL demonstrates a significant improvement in reconstructing over mGANprior clearly being able the error bars. Our method yields a increase in performance with respect to PSNR of 2.5 dB's at 50\%, 2.23 dB's 30\%, 1.99 dB's at 20\%, and 1.58 dB's at 10\% observed pixels. For 1\% of observed pixels these both methods yield similar results. Lastly, qualitative results can be seen on Figure \ref{fig:qual_inp_mgan}.\underline{Super-Resolution}- mGAN-RTIL show substantial improvement with image reconstruction at downsampling factor of $\frac{1}{2}$ and $\frac{1}{4}$, increase of 2.25 db's and 1.51 dB's respectively. However, at $\frac{1}{16}$ down-sampling factor mGAN out performs RTIL-mGAN slightly, but within the error bars. For qualitative results please refer to the appendix Figure~\ref{fig:mgan_sr}. 
\begin{figure*}[ht!]
	\begin{center}
	\includegraphics[width=1\linewidth,height=6cm]{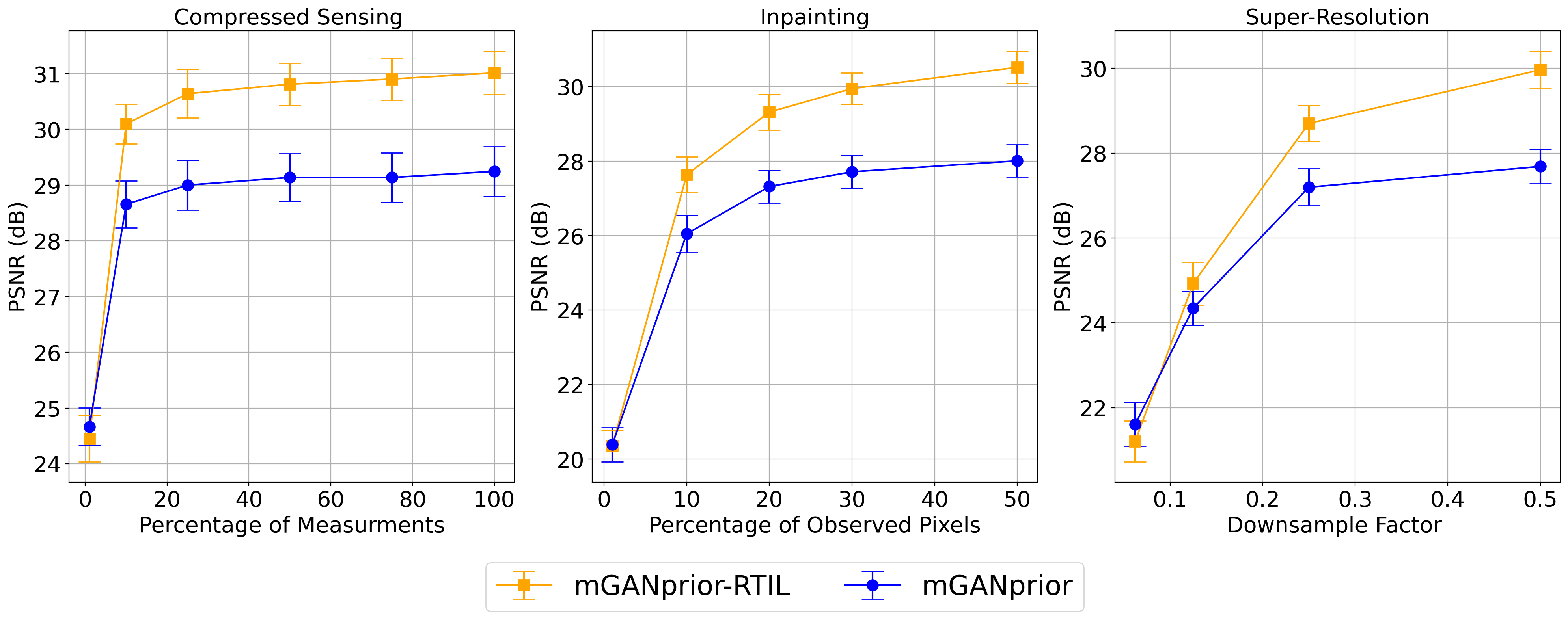}
	\end{center}
	\vspace{-.5cm}
	\caption{Performance of mGANprior-RTIL and vanilla trained mGAN for compressed sensing, inpainting, and super resolution for various under-sampling ratios. mGAN-RTIL increases performances over vanilla mGAN with respect to PSNR over each under-sampling ratio 
	by a noticeable margin, except for super-resolution problems at low under-sampling ratio's. The vertical bars indicate 95\% intervals. }
	\label{fig:RTIL-vs-mGAN}
\end{figure*}
\subsection{Ablation Study}
This section provides two ablations studies, comparing vanilla training versus RTIL for each intermediat layer trained using ILO-RTIL and mGANprior-RTIL. Each experiment was test on 5 images sampled randomly from CelebA-HQ. 
\subsubsection{ILO-RTIL}
\label{sub:ilo-rtil}
All results refer to Figure \ref{fig:abl-ILO}, recall Figure \ref{fig:rtil_ILO} and Alogirthm \ref{alg:ILO} for notation, $z_{0}$ refers to optimizing over the initial latent vector~\citep{bora2017compressed}, $z_{1}$ denotes sequentially optimizing over the first two layers, and so until representation error  $z_{4}$.  

\underline{Compressed Sensing} - Overall, optimizing up to $z_{4}$ achieves the best performance for ILO-RTIL and ILO, in most under-sampling ratio's it outperforms and at worst ties compared to other intermediate optimization layers. Comparing ILO-RTIL to ILO, there is increase in reconstruction performance on average across each under-sampling for $z_{4}$ of 1.41 dB's, $z_{3}$ .81 dB's, $z_{2}$ .4 dB's, and $z_{1}$ .86 dB's. \underline{Inpaiting}-Refer to appendix for inpaiting results in Figure~\ref{fig:abl-ILO}.
\subsubsection{mGANprior-RTIL}
\label{sub:mgan-abl}
\underline{Compressed Sensing} -
mGANpior-RTIL outperformed vanilla mGANpior in each optimization setting using $N=\{1,10,20\}$, marginal at $N=1$, but more noticeable at $N=\{10,20\}$. Overall $N=20$ achieves the  best performance for both mGANprior and mGANprior-RTIL training across each under-sampling regime.As well as mGANprior-RTIL increases performance on average across all the under-sampling experiments with $N=10$ compared to $N=20$ with mGANprior \underline{Inpaiting}-Refer to appendix for inpaiting results in Figure~\ref{fig:abl-mGAN}.

\vspace{-.1cm}

\begin{figure}[h!]
	\begin{center}
	\includegraphics[width=1\linewidth,height=8cm]{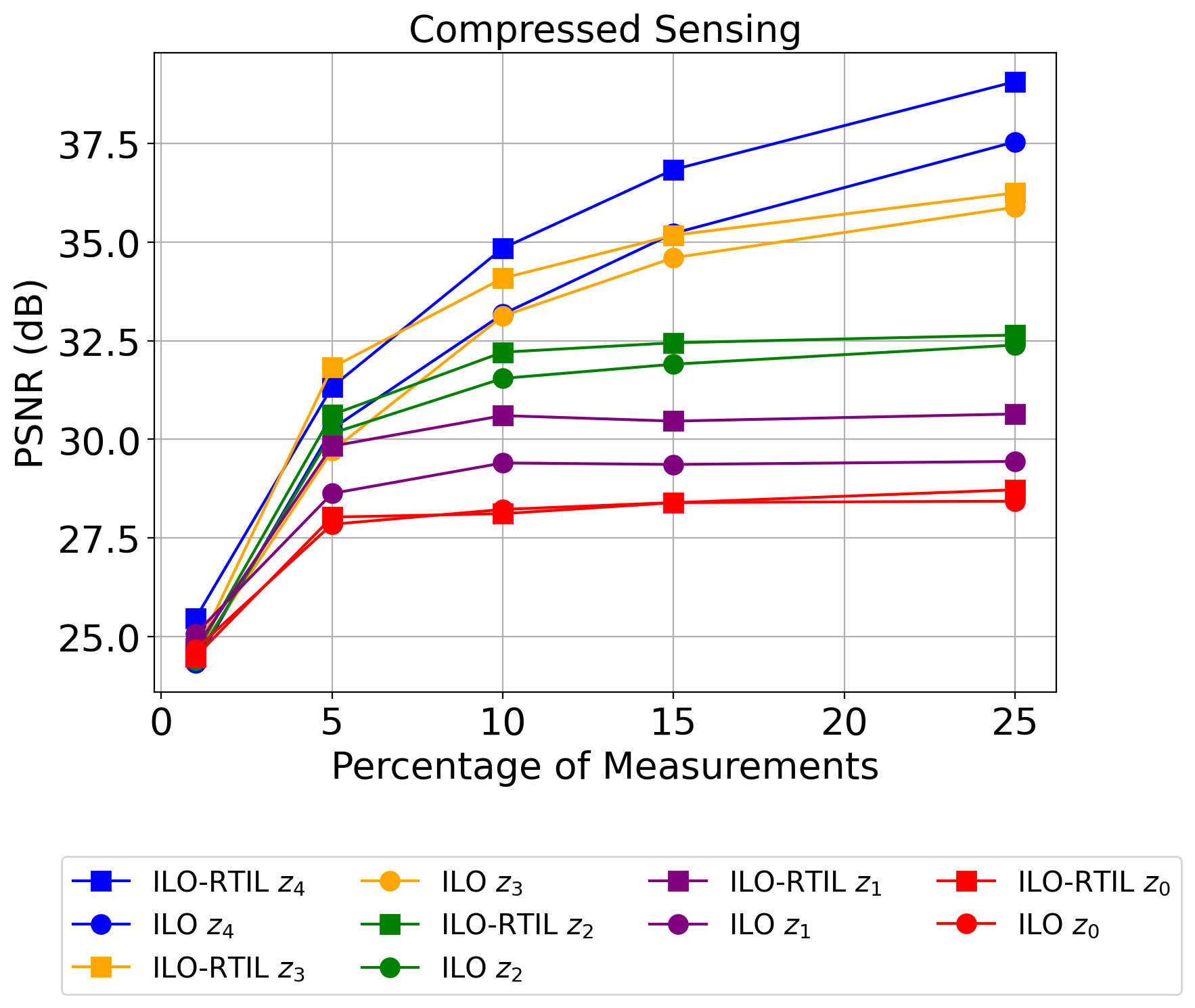}
	\end{center}
	\vspace{-.5cm}
	\caption{Comparing ILO-RTIL to ILO optimizing over various number of  intermediate layers. }
	\label{fig:abl-ILO}
\end{figure}
\section{Theoretical Model for Compressed Sensing with RTIL}

In this section, we examine a simple theoretical model for compressive sensing with RTIL.  To provide the most easily understandable context, we consider the case of a generative model given by a two-layer linear neural network.  Further, we consider the case where the generative model is trained in a supervised manner, and we consider the regime of infinite training data. {The supervised setting allows us to avoid discussing the different methods for learning generative models,  and working in the infinite data regime we avoid statistical estimation errors}.


We assume that the true signal distribution is given by
$x = G^\star(z_0, z_1) = \Wstar_1( \Wstar_0 z_0 + z_1 )$, 
where 
$G^\star: \R^{n_0} \times \R^{n_1} \to \R^{n_d}$, $z_0 \sim \mathcal{N}(0,I_{n_0})$ and $z_1 \sim \mathcal{N}(0,I_{n_1})$ drawn independently, and 
  where $\Wstar_1 \in \R^{n_d \times n_1}$ and $\Wstar_0 \in \R^{n_1 \times n_0}$. We moreover assume that $n_0 < n_1 < n_d$, and $\Wstar_0, \Wstar_1$ are full rank.  We furthermore assume that $\Wstar_0$ is known but that $\Wstar_1$ is unknown.
 
\subsection{Training the Models}  
We consider the training of two generative models.  The first model is analogous to vanilla GAN training without RTIL.  In it, $\Gvan (z_0) = \Wvan_{1} ( \Wstar_0 z_0)$.  
The second model is analogous to GAN training with RTIL.  In it, $\Grtil (z_0, z_1) = \Wrtil_{1} ( \Wstar_0 z_0 + z_1)$. 

Training each model consists of estimating $\Wvan_1$ and $\Wrtil_{1}$ under a least squares loss.  As we consider the idealized regime of infinite training data, we obtain the following estimates:
\begin{equation}\label{eq:W11}
    \Wvan_1 \in \argmin_{{W_1} \in \R^{n_d \times n_1}} \E_{z_0, z_1} \| G^\star(z_0, z_1) - {W_1} \Wstar_0 z_0 \|_2^2,
\end{equation}
and
\begin{equation}\label{eq:W12}
    \Wrtil_1 \in \argmin_{{W_1} \in \R^{n_d \times n_1}} \E_{z_0, z_1} \| G^\star(z_0, z_1) - {W_1}(\Wstar_0 z_0 + z_1) \|_2^2.
\end{equation}

The next lemma is a simple consequence of the property of the Gaussian distribution. 
\begin{lemma}\label{lemma:training} 
Let $\Wvan_1$ satisfy in \eqref{eq:W11} then $\Wvan_1 \Wstar_0 = \Wstar_1 \Wstar_0$. 
Moreover, 
there exists a unique $\Wrtil_{1}$ solution of \eqref{eq:W12}, given by $\Wrtil_{1} = \Wstar_1$.
\end{lemma}
While the above lemma shows that $\Wvan_1$ will equal $\Wstar_1$ on the range of $\Wstar_0$, its behavior on the space orthogonal to the range of $\Wstar_0$ depends on how \eqref{eq:W11} is solved. To simplify the discussion below we will consider $\Wvan_1$ to be the minimum (Frobenius) norm solution of \eqref{eq:W11}. 

\subsection{Compressed Sensing}
We next analyze the compressed sensing problem with the trained generative models.  Let $x^\star = G^\star(z_0^\star, z_1^\star)$ be an unknown signal, where $z_0^\star \sim \mathcal{N}(0,I_{n_0})$ and $z^\star_1 \sim \mathcal{N}(0,I_{n_1})$. We consider the problem of recovering $x^\star$ given compressed linear measurements $y = A x^\star \in \R^{m}$ where $n_1~\leq~m~<~n_d$. When using the base model $\Gvan$ and the higher dimensional model $\Gvantil$, we solve
\begin{equation*}
\begin{aligned}
    \zvan_{0} &= \argmin_{\hat{z}_0 \in \R^{n_0}} \| y - A \Gvan (\hat{z}_0) \|_2^2, \\
    \zvan_{1} &= \argmin_{\hat{z}_1 \in \R^{n_1}} \| y - A \Gvantil ( \zvan_{0} , \hat{z}_1)\|_2^2. \\
\end{aligned} 
\end{equation*}
We denote with $\Gvantil$ the vanilla model trained with \eqref{eq:W11}, where $z_1$ is introduced as an intermediate variable, $\Gvantil (z_0, z_1) = \Wvan_1( \Wstar_0 z_0 + z_1)$. When using the model $\Grtil $ trained with RTIL, we solve
\begin{equation*}
\begin{aligned}
    \zrtil_{0} &= \argmin_{\hat{z}_0 \in \R^{n_0}} \| y - A \Grtil (\hat{z}_0, 0) \|_2^2, \\
    \zrtil_{1} &= \argmin_{\hat{z}_1 \in \R^{n_1}} \| y - A \Grtil ( \zrtil_{0}  , \hat{z}_1)\|_2^2. \\
\end{aligned} 
\end{equation*}
Moreover, we take $\zrtil_{1}$ to be the minimum norm solution. As shown in the proof $\zvan_{0}$ and $(\zrtil_{0}, \zrtil_{1})$ are unique.
Notice that $(\zvan_{0},\zvan_{1})$ and $(\zrtil_{0}, \zrtil_{1})$ correspond to the latent vector and intermediate layer variables obtained applying the ILO Algorithm \ref{alg:ILO} with inputs  $\{\Gvantil, A, y \}$ and $\{\Grtil, A, y \}$ respectively. 

The following lemma quantifies the reconstruction errors made by using the two generative models. 

\begin{lemma}\label{lemma:CS}
Assume that $\Wstar_0, \Wstar_1$ are full rank. Let $\Wvan_1$ be the minimum norm solution of \eqref{eq:W11} and $\Wrtil_1$ be the solution of \eqref{eq:W12}. Let $A \in \R^{m \times n_d}$ with i.i.d. $\mathcal{N}(0,1)$ entries. Then with probability $1$
\begin{multline}\label{eq:ERR1}
    \E_{z^\star_0, z^\star_1} \big[\| G^\star(z^\star_0, z^\star_1) -  \Gvantil (\zhat_0, \zhat_1) \|^2 \big] \\ \geq \max_{h \in \text{range}(\Wstar_0)^\perp}\|(I_{n_d} - \mathcal{P}_{\Wstar_1 \Wstar_0} \big) \Wstar_1 h\|_2^2 > 0 
\end{multline}
and 
\begin{equation}\label{eq:ERR2}
    \E_{z^\star_0, z^\star_1} [\| G^\star(z^\star_0, z^\star_1) -  \Grtil (\ztil_0, \ztil_1) \|^2] =0.
\end{equation}
\end{lemma}
The above results illustrate why training generative models using RTIL can enable better compressed sensing performance.  In the simplified setting of a two-layer linear neural network trained in a supervised manner with a known first layer and in the infinite data regime, we see that the lower dimensional generative model incurs error in the second layer's weights in the orthogonal complement of the range of the first layer.  This results in an increase in the reconstruction error when solving compressed sensing.



\section{Conclusion}

We have introduced a principle for training GANs that are intended to be used for solving inverse problems.  That principle states that if the inversion algorithm optimizes over intermediate layers of the network, then during training the network should be regularized in those layers. We instantiate this principle for two recent and successful optimization algorithms, Intermediate Layer Optimization~\citep{ILO2021} and  the Multi-Code Prior~\citep{Gu_2020_CVPR}.  For both of these algorithm, we devise a new GAN training algorithm.   Empirically, we show our trained GANs allow better reconstruction in compressed sensing, inpainting, and super resolution across multiple under-sampling regimes, when compared to GANs trained in a vanilla manner. We note that our methodology only applies in the case of inversion methods that optimize over intermediate layers. However, there has been multiple competitive methods of this form published recently, each of which we show can benefit from this approach.  Tools like those proposed in this paper, provided sufficient computational resources, may allow these methods to be even more competitive in real-world applications. 

\newpage

\newpage
\appendix
\onecolumn
\begin{alphasection}
\section{Appendix}

Code provided in supplementary folder. Computational requirements for this paper are two NVIDIA 2080 Ti GPU: training StyleGAN2 uses both GPU's and one GPU for inversion, training PGGAN requires one GPU and one for inversion. 
\subsection{RTIL-ILO Training Details }

All experiments for ILO inversion method used StyleGAN2 architecture~\citep{karras2020analyzing}, both vanilla and RTIL models were trained for 700,000 iterations using the same training parameters, i.e., learning rate, batch size, and regularization updates. 
Below table \ref{table:stylegan2} outlines the macroscopic view of the StlyeGAN2 architecture. Please refer StlyeGAN2 paper~\citep{karras2020analyzing} for more architecture details, between each convolutional layer there normalization operation called weight demodulation. During training for RTIL the distribution was induced after each block in the network, which corresponds to a cells 2-6 in table \ref{table:stylegan2} up to 4-th convolutional block. 

All experiments for ILO inversion method used StyleGAN2 architecture~\citep{karras2020analyzing}, both vanilla and RTIL models were trained for 700,000 iterations using the same training parameters, i.e., learning rate, batch size, and regularization updates. 
Please refer to the code for more details on the training process and the StlyeGAN2 paper~\citep{karras2020analyzing} for more details on the architecture. 
Below Table \ref{table:stylegan2} outlines the macroscopic view of the StlyeGAN2 architecture, between each convolutional layer there is a normalization operation called weight demodulation. During training for RTIL the additional latent variables were added after each block in the network, which correspond to cells 2-6 in Table \ref{table:stylegan2} up to 4-th convolutional block. 

\begin{table}[h!]
\centering
\caption{StyleGan2 for image size $256 \times 256 \times 3$}
\begin{tabular}{ | p {5cm} | p {3cm} | p {3cm} | p {3cm} |}
\hline
\multicolumn{3}{ | c | }{\textbf{Generator} } \\
\hline
Operation & Activation & Output Shape \\
\hline
Latent Vector & None & $512 \times$ 1 $\times 1$ \\
$8\times $ MLP (Mapping Network) & LRelu & $512 \times 14$\\

\hline
Constant input & None & $512 \times 4 \times 4$ \\
Conv $3\times 3$ & LRelu & $256 \times 4 \times 4$\\

\hline
Upsample & None & $256 \times 8 \times 8 $ \\
Conv $3\times 3$ & LRelu & $256 \times 8 \times 8$\\
Conv $3\times 3$ & LRelu & $256 \times 8 \times 8$\\
\hline
Upsample & None & $256 \times 16 \times 16 $ \\
Conv $3\times 3$ & LRelu & $256 \times 16 \times 16$\\
Conv $3\times 3$ & LRelu & $256 \times 16 \times 16$\\

\hline
Upsample & None & $256 \times 32 \times 32 $ \\
Conv $3\times 3$ & LRelu & $256 \times 32 \times 32$\\
Conv $3\times 3$ & LRelu & $256 \times 32 \times 32$\\
\hline
Upsample & None & $256 \times 64 \times 64 $ \\
Conv $3\times 3$ & LRelu & $256 \times 64 \times 64$\\
Conv $3\times 3$ & LRelu & $256 \times 64 \times 64$\\
\hline
Upsample & None & $128 \times 128 \times 128 $ \\
Conv $3\times 3$ & LRelu & $128 \times 128 \times 128$\\
Conv $3\times 3$ & LRelu & $128 \times 128 \times 128$\\
\hline
Upsample & None & $128 \times 256 \times 256 $ \\
Conv $3\times 3$ & LRelu & $64 \times 256 \times 256$\\
Conv $3\times 3$ & LRelu & $64 \times 256 \times 256$\\
Conv $1 \times 1$ & Linear & $3 \times 256 \times 256$ \\
\hline
\multicolumn{3}{ | c | }{\textbf{Trainable Parameters} : 12,300,877 } \\
\hline 
\end{tabular}
\label{table:stylegan2}
\end{table}
\newpage

\begin{figure}[h!]
	\begin{center}
	\includegraphics[width=.9\linewidth]{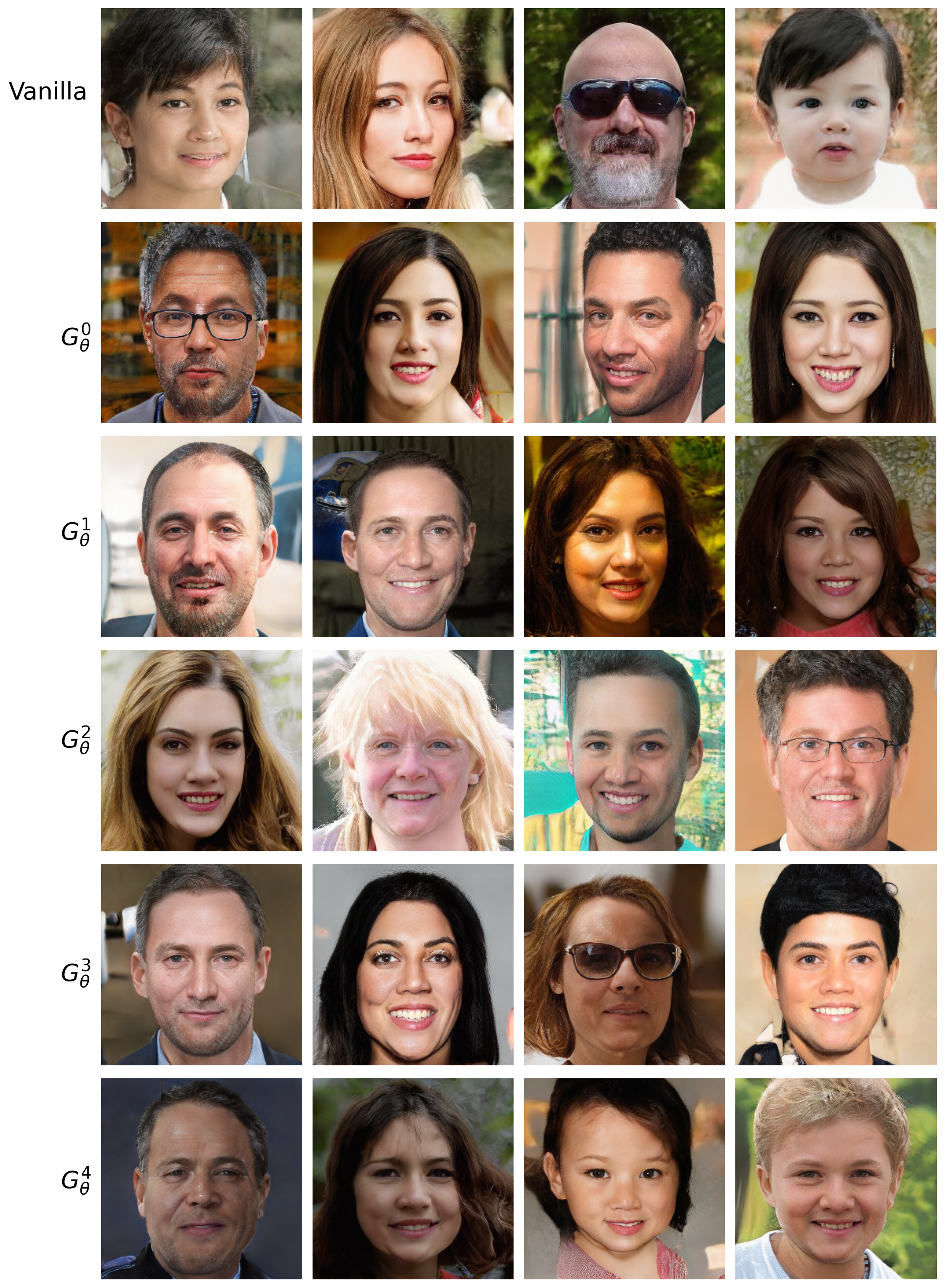}
	\end{center}
	\caption{Samples from each of our trained generative models. Top row consist samples from vanilla trained StyleGAN2, then rows $G_{\theta}^{0} \cdots G_{\theta}^{4}$ are samples from family of generative models trained with RTIL. }
	\label{fig:samples}
\end{figure}

\newpage
\subsection{RTIL-mGANprior Training Details}
\begin{figure*}[H]
\def\svgwidth{\linewidth}
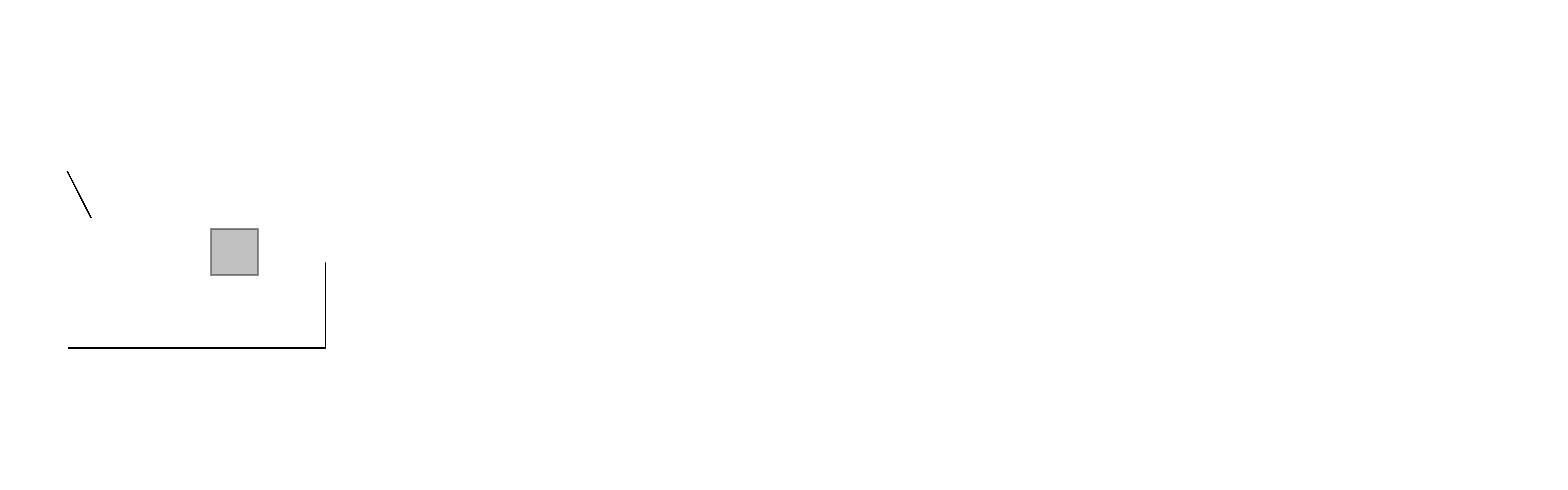
\caption{The left side portrays vanilla training for mGANprior, and the right side demonstrates RTIL. For  mGANprior-RTIL this example with $N = 20$ latent codes, where the top model trains a vanilla model analogous to model on the left, then parameter shares with the model below. }
\label{fig:rtil_mgan}
\end{figure*}

All experiments for mGANprior inversion method used PGGAN Architecture~\citep{karras2018progressive}, both vanilla and RTIL models were trained for 700,000 iterations and use the same training parameters, i.e., learning rate, batch size, and regularization updates. If interested please to refer to the code for more details. Below table \ref{table:PGGAN} outlines the details of the PGGAN architecture. The intermediate latent variable were added after the 4-th block, which corresponds to the 4-th cell in Table \ref{table:PGGAN}. 
\begin{table}[h!]
\centering
\caption{PGGAN for image size $256 \times 256 \times 3$}
\begin{tabular}{ | p {5cm} | p {3cm} | p {3cm} | p {3cm} |}
\hline
\multicolumn{3}{ | c | }{\textbf{Generator} } \\
\hline
Operation & Activation & Output Shape \\
\hline
Latent Vector & None & $512 \times$ 1 $\times 1$ \\
Conv $4\times 4$ & LRelu & $256 \times 4 \times 4$\\
Conv $3\times 3$ & LRelu & $256 \times 4 \times 4$\\
\hline
Upsample & None & $256 \times 8 \times 8 $ \\
Conv $3\times 3$ & LRelu & $256 \times 8 \times 8$\\
Conv $3\times 3$ & LRelu & $256 \times 8 \times 8$\\
\hline
Upsample & None & $256 \times 16 \times 16 $ \\
Conv $3\times 3$ & LRelu & $256 \times 16 \times 16$\\
Conv $3\times 3$ & LRelu & $256 \times 16 \times 16$\\
\hline
Upsample & None & $256 \times 32 \times 32 $ \\
Conv $3\times 3$ & LRelu & $256 \times 32 \times 32$\\
Conv $3\times 3$ & LRelu & $256 \times 32 \times 32$\\
\hline
Upsample & None & $256 \times 64 \times 64 $ \\
Conv $3\times 3$ & LRelu & $128 \times 64 \times 64$\\
Conv $3\times 3$ & LRelu & $128 \times 64 \times 64$\\
\hline
Upsample & None & $128 \times 128 \times 128 $ \\
Conv $3\times 3$ & LRelu & $64 \times 128 \times 128$\\
Conv $3\times 3$ & LRelu & $64 \times 128 \times 128$\\
\hline
Upsample & None & $64 \times 256 \times 256 $ \\
Conv $3\times 3$ & LRelu & $64 \times 256 \times 256$\\
Conv $3\times 3$ & LRelu & $64 \times 256 \times 256$\\
Conv $1 \times 1$ & Linear & $3 \times 256 \times 256$ \\
\hline
\multicolumn{3}{ | c | }{\textbf{Trainable Parameters} : 7,445,443 } \\
\hline 
\end{tabular}
\label{table:PGGAN}
\end{table}

\begin{figure*}[H]
	\begin{center}
	\includegraphics[width=1\linewidth]{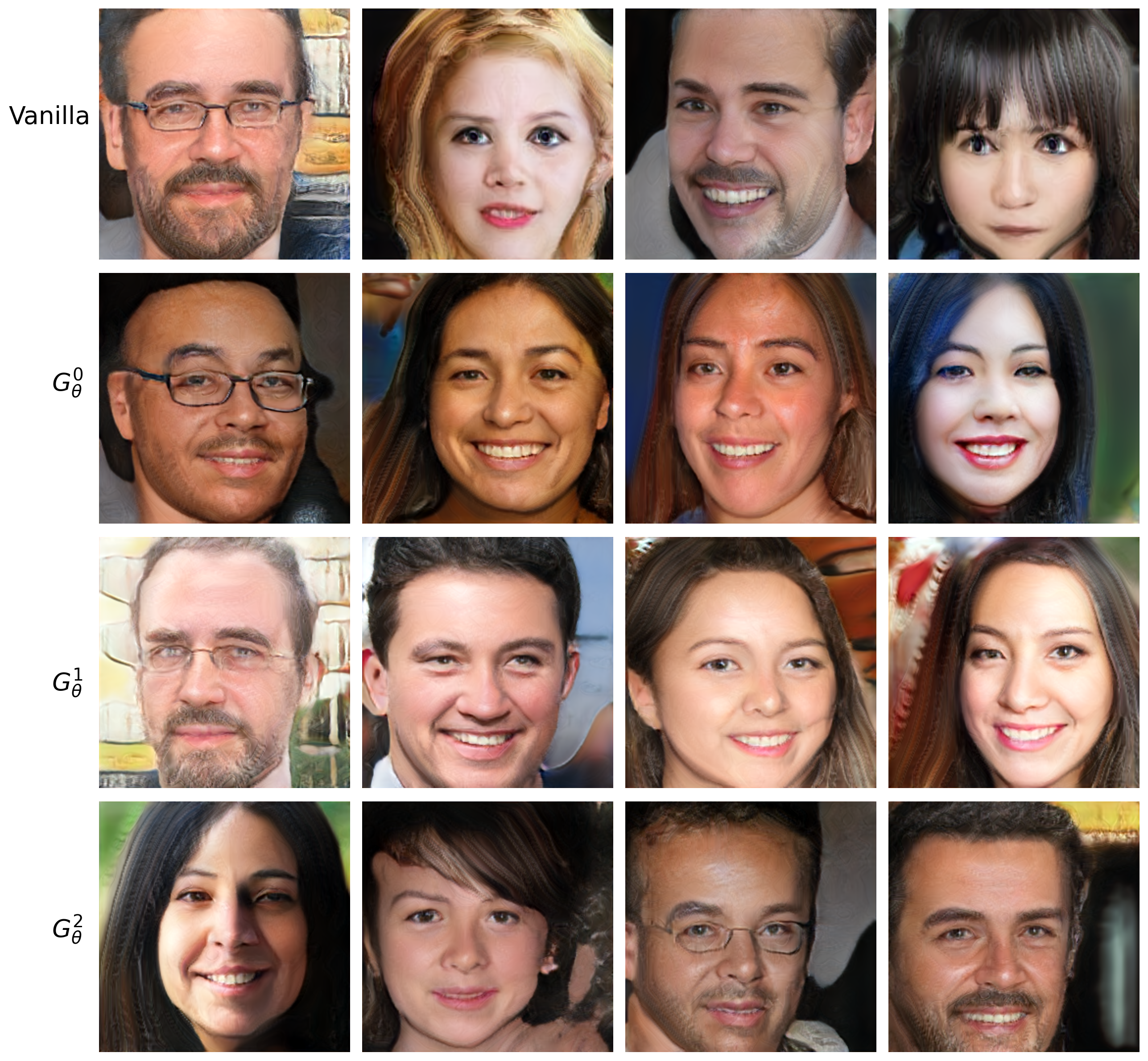}
	\end{center}
	\caption{Samples from each of our trained generative models. Top row consist samples from vanilla trained PGGan, then rows $G_{\theta}^{0} \cdots G_{\theta}^{2}$ are samples from family of generative models trained with RTIL.}
	\label{fig:samples}
\end{figure*}

\newpage

\subsection{Inversion Details for ILO}
 Hyper-parameters were tuned based on indications in the appendix of the ILO paper~\citep{ILO2021} and the official implementation. Specifically, learning rate and number of iterations per layer were tuned for compressed sensing, then the same configuration were used for inpainting and super-resolution. For experiments in Section \ref{sec:exp-ILO-RTIL} Figure \ref{fig:RTIL-vs-ILO} the configuration for the number of iterations per layer is  $\{2000,1000,1000,1000,2000\}$. ILO-RTIL uses a learning rate that begins at $.2$ at the initial layer, then ramps up linearly and ramps down using a cosine scheduler, as proposed by~\citep{karras2020analyzing}. As for ILO, each layer initialized with a learning rate of  $.1$ then optimized independently using the same learning rate scheduler, which is proposed in official Github repository~\citep{ILO2021}. We choose the intermediate layer up to which optimize to, based on ablation study Section \ref{sub:ilo-rtil} Figure \ref{fig:abl-ILO}. 
 Below we report the loss function used for each inverse problem in case of ILO-RTIL and RTIL: 
 \begin{itemize}
     \item \textit{Compressed Sensing} - Mean square error \\
     \item \textit{Inpainting}- Equal weighted combination of mean square error and LPIPS~\citep{zhang2018perceptual} for sufficient number of measurements. Notice that with more than 50\% missing pixels LPIPS did not help reconstruction performance. \\
     \item \textit{Super-Resolution} - Equal weighted combination of mean square error and LPIPS for RTIL models, for vanilla trained models LPIPS was weighted more. This configuration demonstrated increase in performance for all under-sampling ratio's. Best performance occurred when MSE and LPIPS used same image dimensions, i.e no upsampling for the LPIPS loss.For more details refer to GitHub code.
 \end{itemize}

\subsection{Inversion Details for mGANprior}

\begin{figure*}[h!]
\def\svgwidth{\linewidth}
\import{figs/mGANprior}{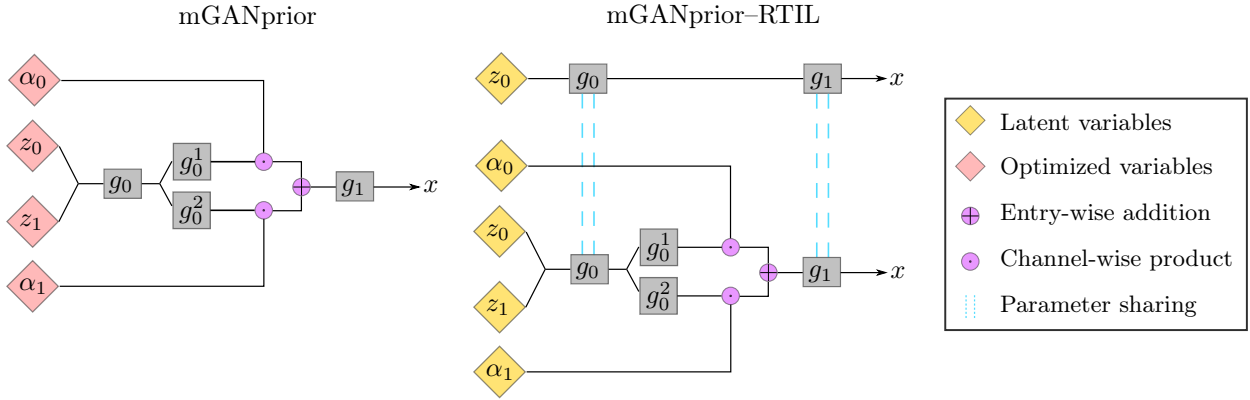}
\caption{The left side portrays vanilla training for mGANprior, and the right side demonstrates RTIL. For  mGANprior-RTIL this example with $N = 2$ latent codes, where the top model trains a vanilla model analogous to model on the left, then parameter shares with the model below. }
\label{fig:rtil_mgan}
\end{figure*}

Hyper-parameters were tuned based on the Github repository~\citep{Gu_2020_CVPR} for Section \ref{sec:RTIL-multi} Figure \ref{fig:RTIL-vs-mGAN}, learning rate and the number of iterations were tuned for compressed sensing then reused for inpainting and super-resolution. mGANprior-RTIL uses Adam~\citep{DBLP:journals/corr/KingmaB14} optimizer initialized at learning rate of $.1$ and optimized for 2500 iterations. mGANprior uses SGD initialized at learning rate $1$ and optimized for 2500 iterations, which is based on the official Github repository~\citep{Gu_2020_CVPR}. Empirically, the mGANprior~\citep{Gu_2020_CVPR} improvement in reconstruction saturated after $N=20$ latent codes. Moreover, for selecting which intermediate layer to optimize over for the vanilla model was determined by the ablation study in Figure \ref{fig:abl-app}. 
 Below we report the loss function used for each inverse problem in case of mGANprior-RTIL and mGANprior
\begin{itemize}
     \item \textit{Compressed Sensing} - Mean square error loss for both methods.\\
     \item \textit{Inpainting}- Mean sqaure error plus $l_{1}$ LPIPS regularization proposed by~\citep{Gu_2020_CVPR}. For mGANprior-RTIL the regularization term was $\lambda = .1$ and for mGANprior $\lambda=.5$. 
     \item \textit{Super-Resolution} - Mean square error plus $l_{1}$ LPIPS regularization proposed by~\citep{Gu_2020_CVPR} or mGANprior-RTIL the regularization term was scaled $\lambda = .1$ and for mGANprior $\lambda=.5$. Best performance occurred when MSE and LPIPS used same image dimensions, i.e no upsampling for the LPIPS loss.
\end{itemize}

\begin{figure*}[h!]
	\begin{center}
	\includegraphics[width=1\linewidth,height=8cm]{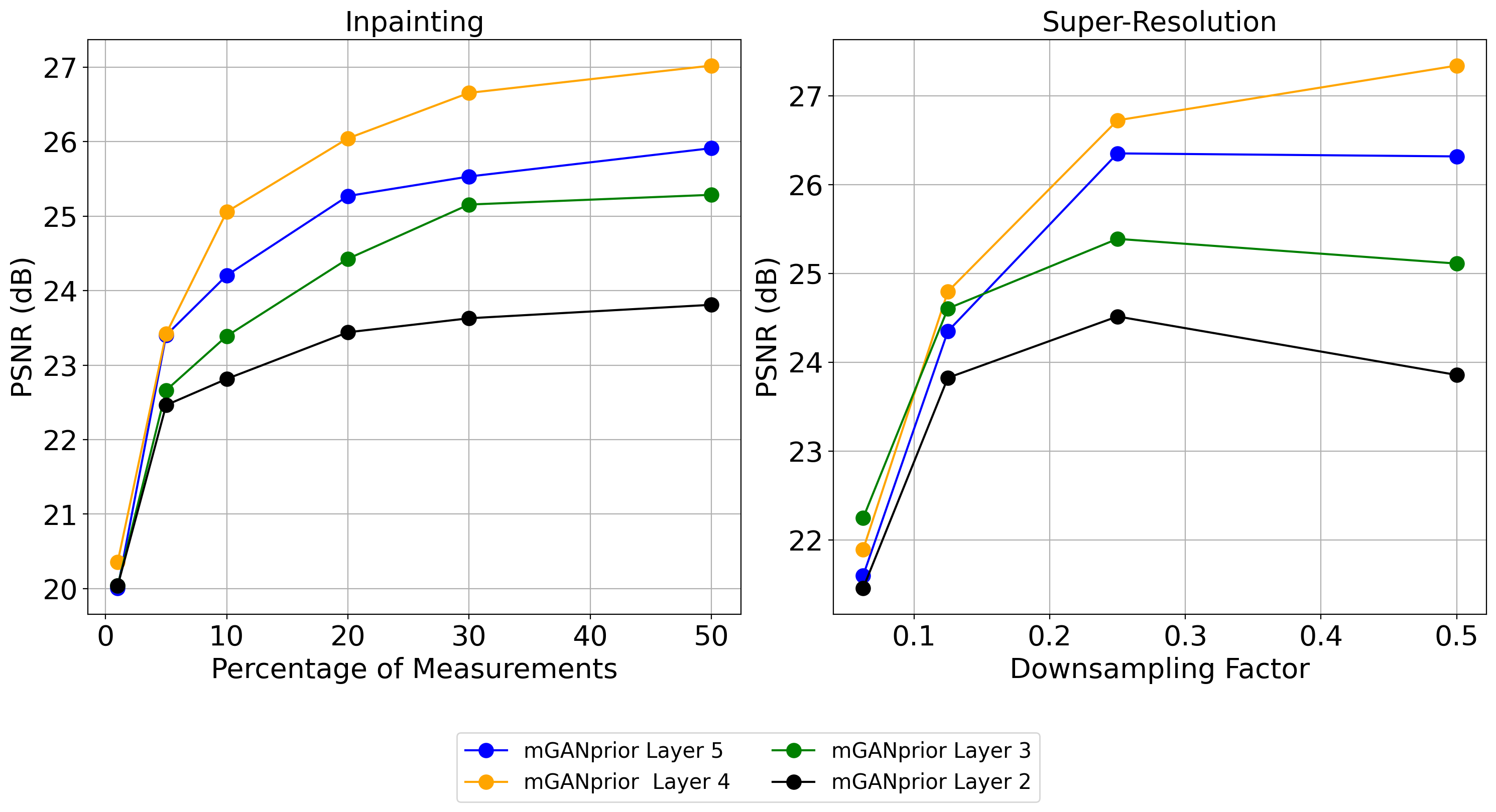}
	\end{center}
	\caption{Effects on  in-painting and super-resolution performance for various intermediate layer on validation set of 5 images.}
	\label{fig:abl-app}
\end{figure*}

\subsection{Inverse Problems Qualitative Results}
\textbf{Inpainting}
\begin{figure}[H]
	\begin{center}
	\includegraphics[width=1\linewidth]{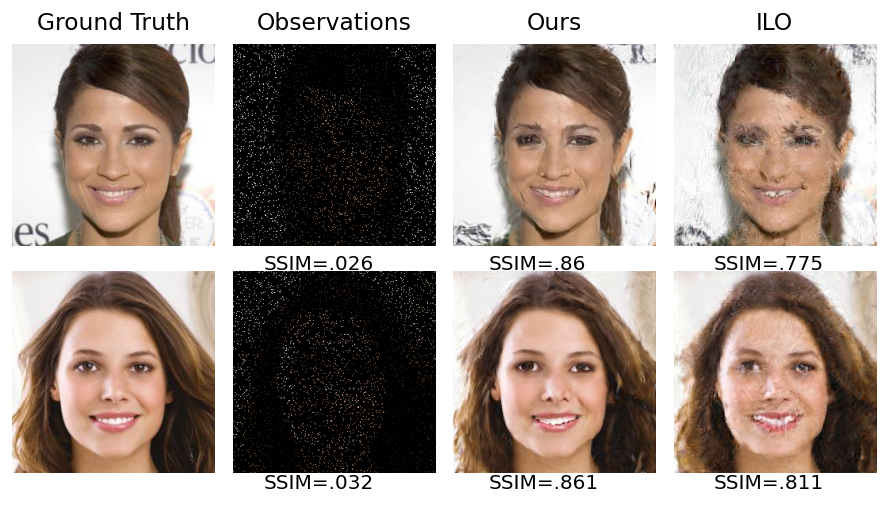}
	\end{center}
	\caption{ Qualitative comparison between our method ILO-RTIL and ILO for inpainting at 5\% of observed pixels. }
	\label{fig:qual_inp_ilo}
\end{figure}
\begin{figure}[H]
	\begin{center}
	\includegraphics[width=1\linewidth]{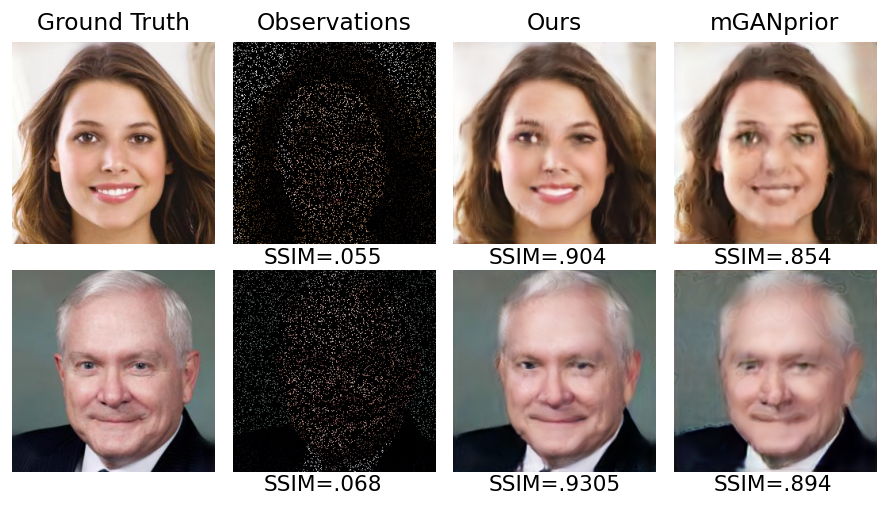}
	\end{center}
	\vspace{-.5cm}
	\caption{ Qualitative comparison between our method mGANprior-RTIL and mGANprior for inpainting at 10\% of observed pixels. }
	\label{fig:qual_inp_mgan}
\end{figure}

\textbf{Super-Resolution}
\begin{figure}[H]
	\begin{center}
	\includegraphics[width=1.0\linewidth]{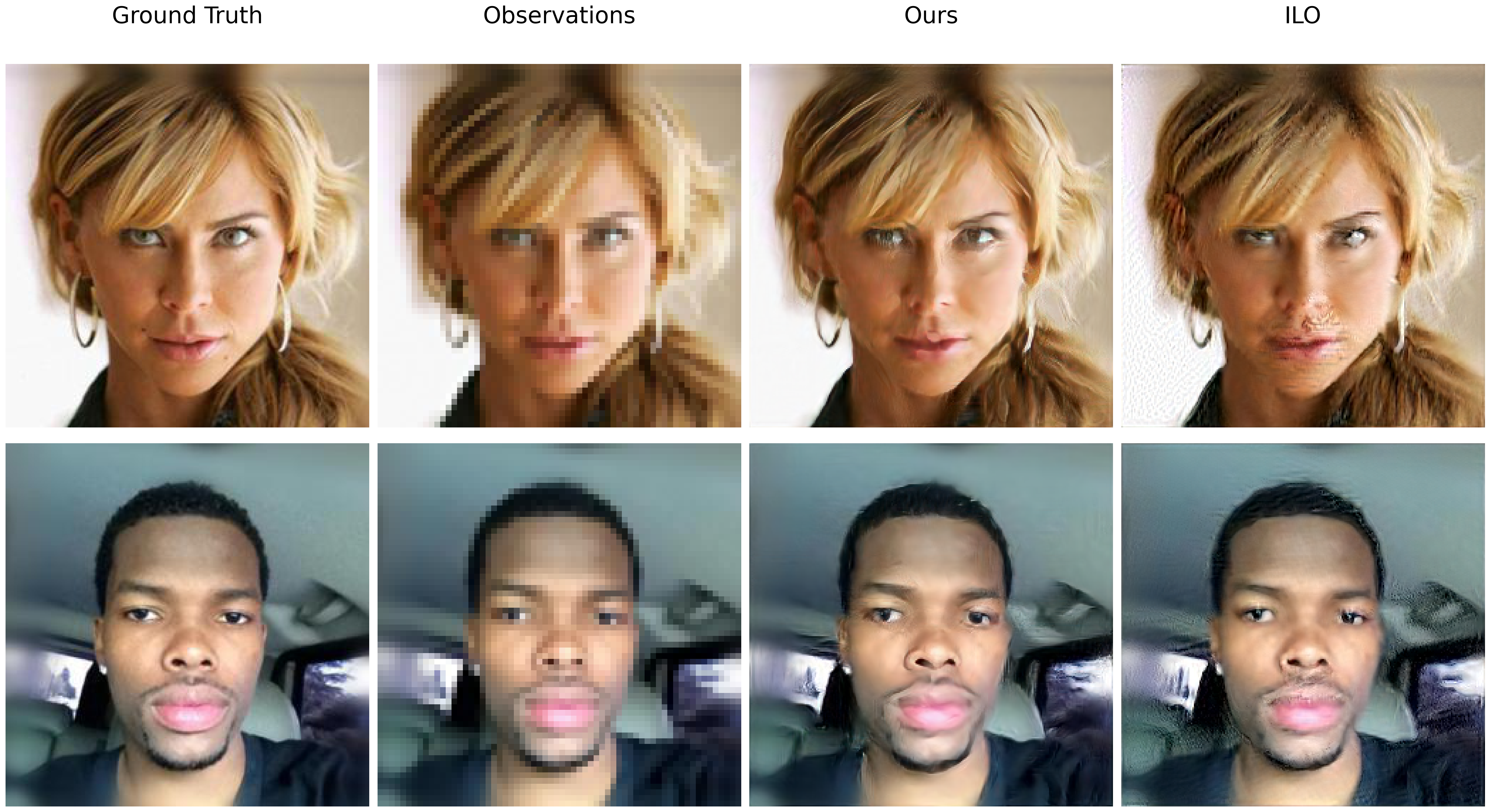}
	\end{center}
	\caption{Comparison between ILO-RTIL to ILO for super-resolution LR 4x (Downsampling factor $\frac{1}{4}$) }
	\label{fig:ilo_sr}
\end{figure}

\begin{figure}[H]
	\begin{center}
	\includegraphics[width=1.0\linewidth]{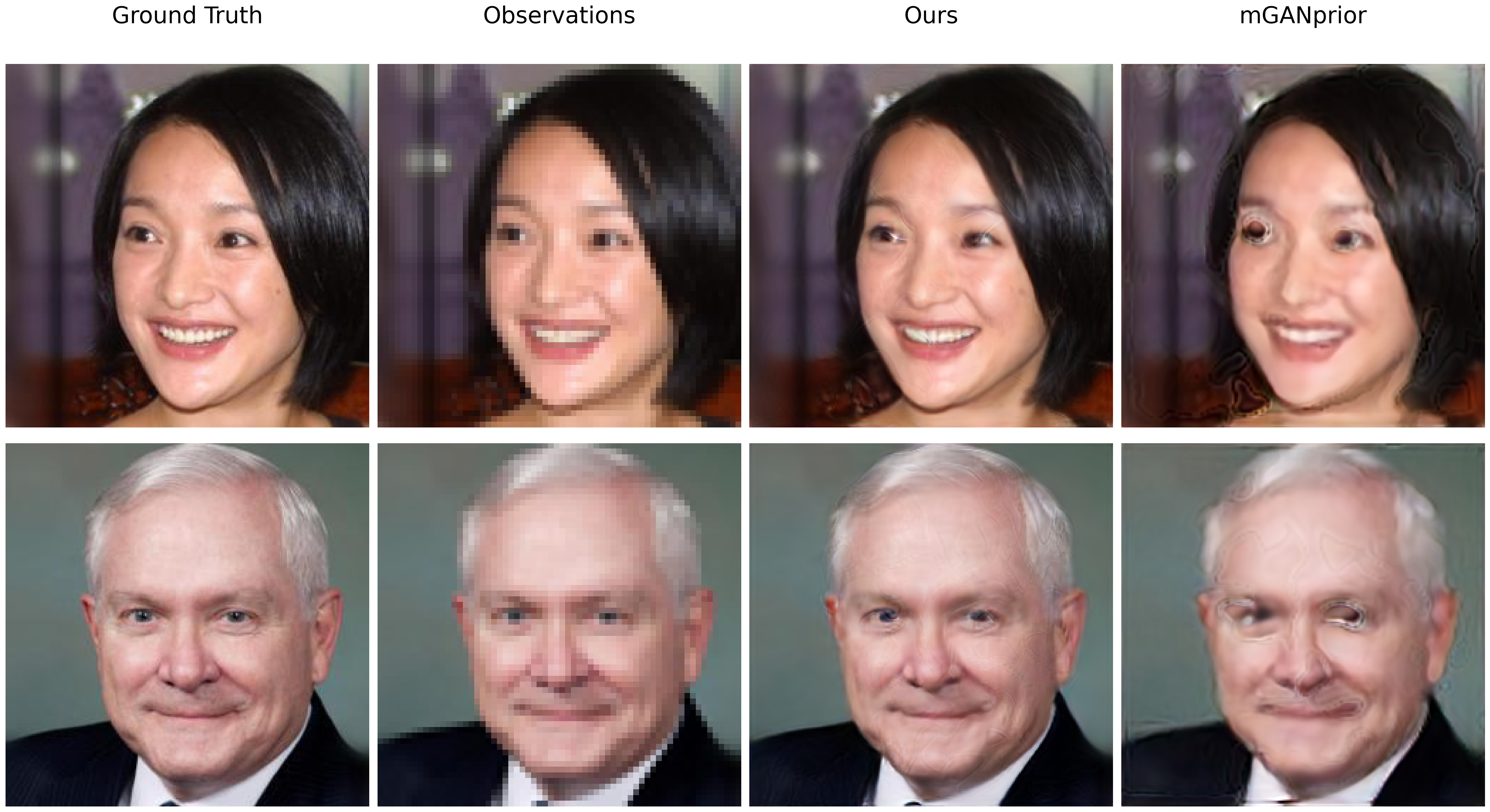}
	\end{center}
	\caption{Comparison between mGANprior-RTIL to mGANprior for super-Resolution LR 4x (Downsampling factor $\frac{1}{4}$)  }
	\label{fig:mgan_sr}
\end{figure}
\newpage
\subsection{Ablation}
\subsubsection{ILO-Ablation}
This section corresponds to Section \ref{sub:ilo-rtil}, Figure \ref{fig:abl-ILO}, the configuration for each optimization setting go as: $z_{0}=\{2000\}$, $z_{1}=\{2000,2000\}$, $z_{2}=\{2000,1000,2000\}$ ,  $z_{3}=\{2000,1000,1000,2000\}$, $z_{4}=\{2000,1000,1000,1000,2000\}$ iterations per layer.
\begin{figure}[H]
	\begin{center}
	\includegraphics[width=1\linewidth,height=8cm]{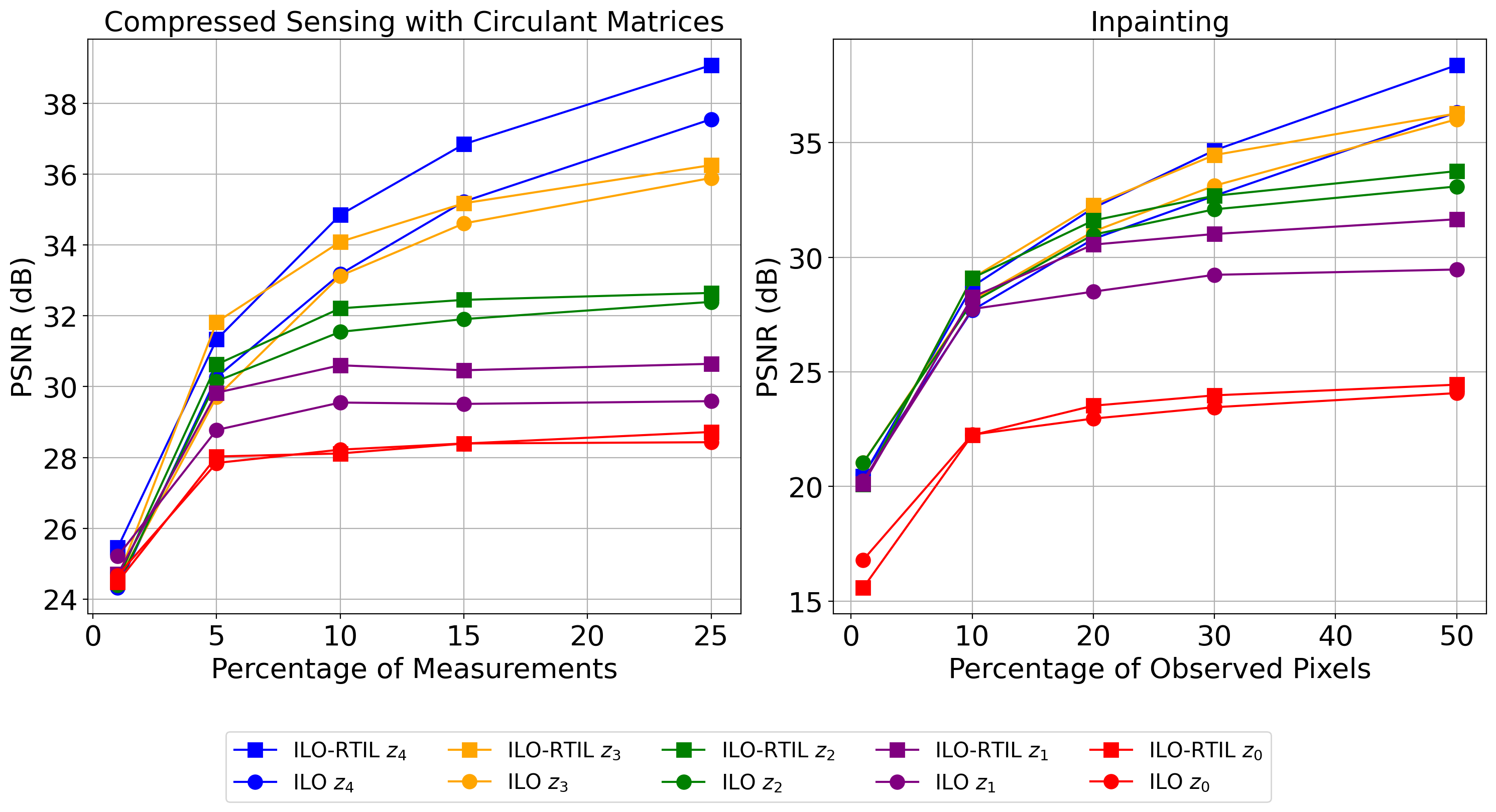}
	\end{center}
	\caption{Comparing compressed sensing performance between ILO-RTIL and ILO for various number of intermediate layers. }
	\label{fig:abl-ilo}
\end{figure}

\subsubsection{mGANprior-Ablation}
This section corresponds to Section \ref{sub:mgan-abl}
\begin{figure}[H]
	\begin{center}
	\includegraphics[width=1\linewidth,height=8cm]{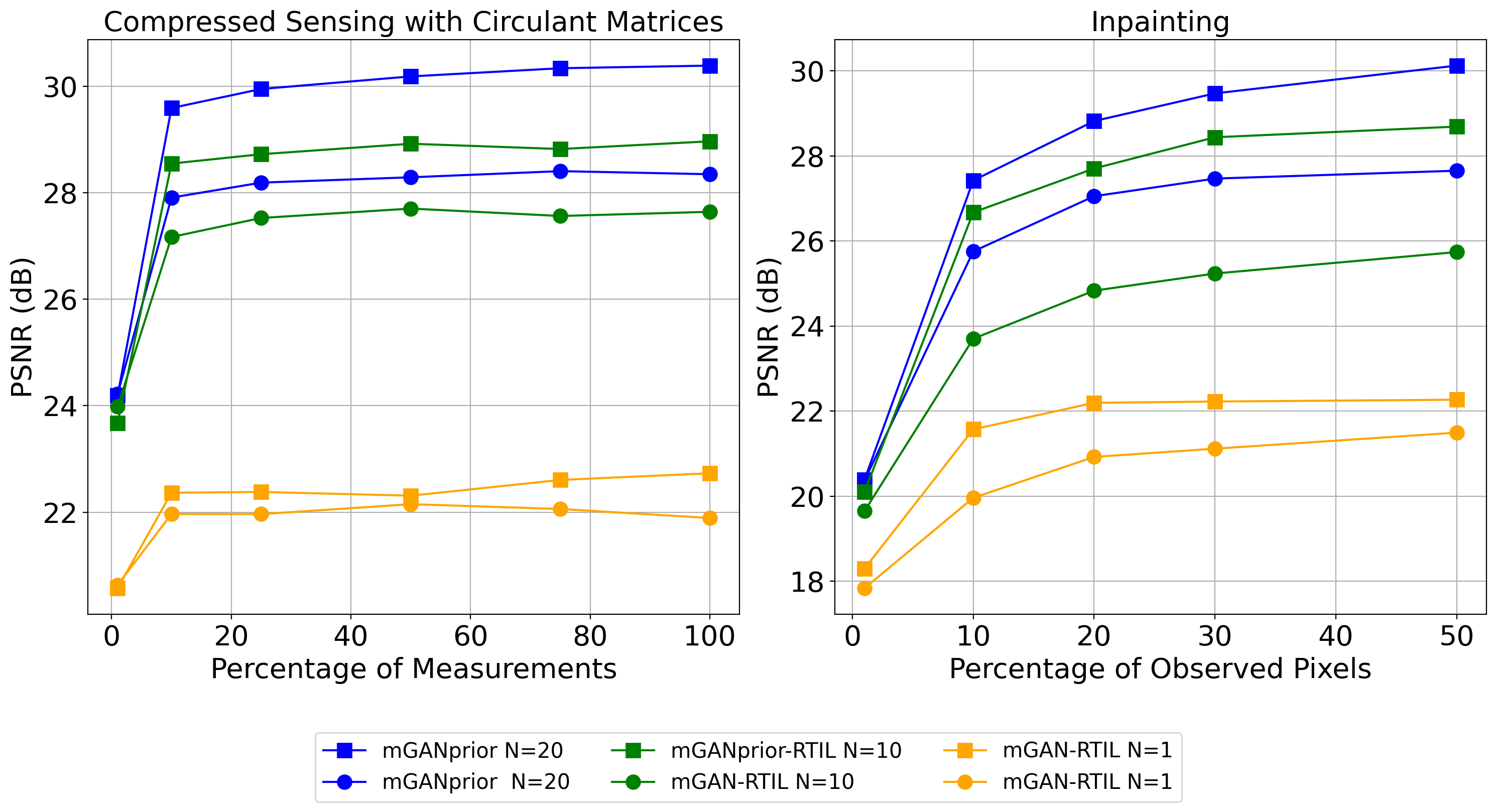}
	\end{center}
	\caption{Comparing compressed sensing performance between mGANprior-RTIL and mGANprior for various number of latent codes. }
	\label{fig:abl-mGAN}
\end{figure}

\end{alphasection}

\newpage
\section{Proof}

\begin{proof}[Proof of Lemma \ref{lemma:training}]
Notice that for any $W_1 \in \R^{n_d \times n_1}$
\[
    \begin{aligned}
    \E_{z_0, z_1} [\| \Wstar_1 (\Wstar_0 z_0 + z_1) - W_1 \Wstar_0 z_0  \|_2^2]
    &= \E_{z_0} [\| (\Wstar_1 - W_1) \Wstar_0 z_0 \|_2^2] + \E_{z_1} [\| \Wstar_1 z_1 \|_2^2]\\
    &= \| (\Wstar_1 - W_1 )\Wstar_0 \|_F^2 + \|\Wstar_1\|^2_F,
    \end{aligned}
\]
where the first equality used independence of $z_0$ and $z_1$ and the second used the fact that $\E_{z \sim \mathcal{N}(0,I_n)}\| M z\|_2^2 = \|M\|_F^2$ for any matrix $M \in \R^{m \times n}$.
The solutions of \eqref{eq:W11} are then solutions of 
\begin{equation}\label{eq:calF}
     \min_{W_1 \in \R^{n_1 \times n_0}} \| \mathcal{L}(\Wstar_1) - \mathcal{L}(W_1)\|_F^2
\end{equation}
where $\mathcal{L}: \R^{n_d \times n_1} \to \R^{n_1 \times n_0}$ is the linear operator given by $\mathcal{L}: W_1 \mapsto W_1 \Wstar_0$. Since  $n_1 > n_0$, this operator is singular and there are infinite solutions of \eqref{eq:calF}, all of which satisfy $\mathcal{L}(\Wstar_1) = \mathcal{L}(W_1)$, i.e. the thesis.

Regarding \eqref{eq:W12}, observe instead that for any $W_1 \in \R^{n_d \times n_1}$ 
\[
\begin{aligned}
\E_{z_0, z_1} [\| \Wstar_1 (\Wstar_0 z_0 + z_1) - W_1 (\Wstar_0 z_0 + z_1)  \|_2^2]
&= \| (\Wstar_1 - W_1) \Wstar_0 \|_F^2 + \|\Wstar_1 - W_1\|^2_F.
\end{aligned}
\]
This shows $W_1 = \Wstar_1 $ is the unique minimizer of \eqref{eq:W12}.
\end{proof}

Before analyzing the solution of compressed sensing with the trained generative models, we observe the following fact on the minimum norm solution of \eqref{eq:W11}. 

\begin{lemma}\label{eq:minNorm}
Let $\Wvan_1$ be the minimum Frobenius norm solution of \eqref{eq:W11}. Then $\Wvan_1 = \Wstar_1 \Wstar_0 (\Wstar_0)^\dagger$ where $(\Wstar_0)^\dagger$ is the pseudoinverse of $\Wstar_0$.
\end{lemma}
\begin{proof}
    Notice that $\mathcal{L}(\Wstar_1) = \mathcal{L}(\Wvan_1)$, where $\mathcal{L}$ is defined in the proof of Lemma \eqref{lemma:training}.  
    We next show that $\Wvan_1$ is orthogonal to the null space of $\mathcal{L}$, which implies the thesis. 
    
    Let $W_1 \in \R^{n_d \times n_1}$ be such that $\mathcal{L}(W_1) = W_1 \Wstar_0 = 0$. Then we have
    \[
    \begin{aligned}
        \langle \Wvan_1, W_1\rangle_F &= tr((\Wvan_1)^T W_1) \\ 
        &= tr(W_1 (\Wstar_0 (\Wstar_0)^\dagger)^T (\Wstar_1)^T ) \\
        &= tr(W_1 \Wstar_0 (\Wstar_0)^\dagger (\Wstar_1)^T ) \\ 
        &= 0,
    \end{aligned}
    \]
    where third equality uses the fact that $\Wstar_0 (\Wstar_0)^\dagger$ is symmetric and the fourth one the assumption on $W_1$.
\end{proof}

\begin{proof}[Proof of Lemma \ref{lemma:CS}]
\smallskip${}$

- {Proof of \eqref{eq:ERR1}}.
Notice that by the previous lemma $\Wvan_1 = \Wstar_1 \Wstar_0 (\Wstar_0)^\dagger = \Wstar_1 \mathcal{P}_{\Wstar_0}$  where $\mathcal{P}_{\Wstar_0}$ is the orthogonal projector onto the range of $\Wstar_0$. Moreover, with probability 1, $(A \Wstar_1 \Wstar_0 )$ is full rank.
It follows then, that $(\zvan_0, \zvan_1)$ is given by
\[
\begin{aligned}
    \zvan_0 &= \argmin_{z_0 \in \R^{n_0}} \| y - A \Wstar_1 \Wstar_0 z_0 \|_2^2, \\
    \zvan_1 &= \argmin_{z_1 \in \R^{n_1}} \| y - A \Wstar_1 ( \Wstar_0 z_0^{(0)} + \mathcal{P}_{\Wstar_0} z_1 )\|_2^2. 
\end{aligned} 
\]
In particular the minimum norm solution $\zvan_1$ will satisfies $(I - \mathcal{P}_{\Wstar_0})
\zvan_1 = 0$.
We have therefore $\Gvantil(\zvan_0, \zvan_1) =  \Wstar_1 (\Wstar_0 \zstar_0 +  \Wstar_0 \mathcal{M}_1 \zstar_1)$ where $\mathcal{M}_1 = (A \Wstar_1 \Wstar_0 )^\dagger \Wstar_1$.

The reconstruction error is then given by
    \[
    \begin{aligned}
        \E_{\zstar_0, \zstar_1} \| G^\star(\zstar_0, \zstar_1) -  \Gvantil (\zvan_{0}, \zvan_{1}) \|_2^2 
        &=   \| \Wstar_1   - \Wstar_1 \Wstar_0 \mathcal{M}_1 \|_F^2 \\
        &\geq \min_{\mathcal{M} \in \R^{n_0 \times n_1} } \| \Wstar_1   - \Wstar_1 \Wstar_0 \mathcal{M} \|_F^2\\
        &= \|\big(I_{n_d} - \mathcal{P}_{\Wstar_1 \Wstar_0} \big) \Wstar_1 \|_F^2
    \end{aligned}
    \] 
    where the first equality follows from the properties of the normal distribution. Regarding then the minimization problem $\min_{\mathcal{M} \in \R^{n_0 \times n_1} } \| \Wstar_1   - \Wstar_1 \Wstar_0 \mathcal{M}_1 \|_F^2$, notice that this is convex and the critical points satisfy
    \[
         (\Wstar_1 \Wstar_0)^T  \Wstar_1  = (\Wstar_1 \Wstar_0)^T (\Wstar_1 \Wstar_0) \mathcal{M}
    \]
    Using the fact that $\Wstar_1 \Wstar_0$ is full rank, the unique solution is found to be $[ (\Wstar_1 \Wstar_0)^T (\Wstar_1 \Wstar_0) ]^{-1}(\Wstar_1 \Wstar_0)^T  \Wstar_1 $, which gives the last equality. 
    \smallskip
    
    Note now that $\mathcal{P}_{\Wstar_1 \Wstar_0}$ is the projector onto the range of $\Wstar_1 \Wstar_0$ and $\Wstar_1$ is full rank. Thus
    \[
        \|\big(I_{n_d} - \mathcal{P}_{\Wstar_1 \Wstar_0} \big) \Wstar_1 \|_F^2
        \geq \|\big(I_{n_d} - \mathcal{P}_{\Wstar_1 \Wstar_0} \big) \Wstar_1 \|_2^2
        = \max_{h \in \text{range}(\Wstar_0)^\perp}\|\big(I_{n_d} - \mathcal{P}_{\Wstar_1 \Wstar_0} \big) \Wstar_1 h\|_2^2 > 0.
    \]
    \medskip
    
    - \textit{Proof of \eqref{eq:ERR2}}
    Notice again that with probability 1, $A \Wstar_1 \Wstar_0$ and $A \Wstar_1$ have full rank. Moreover $\Grtil (\zrtil_{0}, \zrtil_{1}) =  \Wstar_1 (\Wstar_0 \zrtil_{0} +  \zrtil_{1})$ where 
    \[
    \begin{aligned}
        \zrtil_0 &= \argmin_{z_0 \in \R^{n_0}} \| y - A \Wrtil_1 \Wstar_0 z_0 \|_2^2, \\
        \zrtil_1  &= \argmin_{z_1 \in \R^{n_1}} \| y - A \Wrtil_1 ( \Wstar_0 \zrtil_{0} + z_1 )\|_2^2. \\
    \end{aligned} 
    \]
    It is then easy to see that
    \[
        \begin{aligned}
            \zrtil_0 &= \zstar_0 + (A \Wstar_1 \Wstar_0 )^\dagger A \Wstar_1 \zstar_1 \\
            \zrtil_1 &= \zstar_1 - \Wstar_0 (A \Wstar_1 \Wstar_0 )^\dagger A \Wstar_1 \zstar_1 
        \end{aligned}
    \]
    where $(A \Wstar_1 \Wstar_0 )^\dagger$ denotes the pseudoinverse of $(A \Wstar_1 \Wstar_0 )$. It then follows that $\Wstar_1 (\Wstar_0 \zrtil_0 + \zrtil_1 ) =\Wstar_1 (\Wstar_0 z_0^\star +  z_1^\star) = x^\star$, which implies the thesis.
\end{proof}

\end{document}